\documentclass{article}

\PassOptionsToPackage{numbers,compress}{natbib}

\usepackage[main, final]{neurips_2026} %

\usepackage[utf8]{inputenc} %
\usepackage[T1]{fontenc} %
\usepackage{url} %
\usepackage{booktabs} %
\usepackage{amsfonts} %
\usepackage{amsthm}
\usepackage{nicefrac} %
\usepackage{microtype} %
\usepackage{xcolor} %
\usepackage{iitem}
\usepackage{booktabs}
\usepackage{colortbl}
\usepackage{multicol}
\usepackage{multirow}
\usepackage{hyperref}
\usepackage{graphicx}
\usepackage{adjustbox}
\usepackage{makecell} %
\usepackage{subcaption}

\newcommand{\myheading}[1]{\vspace{1ex}\noindent \textbf{#1}}

\def\veps{{\boldsymbol{\epsilon}}}
\def\tref{\text{ref}}

\usepackage{wrapfig}

\usepackage{amsmath,amsfonts,bm}

\def\eqref#1{equation~\ref{#1}}

\def\1{\bm{1}}

\def\rvc{{\mathbf{c}}}

\def\rvz{{\mathbf{z}}}

\def\vc{{\bm{c}}}

\def\vr{{\bm{r}}}

\def\vw{{\bm{w}}}
\def\vx{{\bm{x}}}
\def\vy{{\bm{y}}}
\def\vz{{\bm{z}}}

\def\mI{{\bm{I}}}

\DeclareMathAlphabet{\mathsfit}{\encodingdefault}{\sfdefault}{m}{sl}
\SetMathAlphabet{\mathsfit}{bold}{\encodingdefault}{\sfdefault}{bx}{n}

\def\gG{{\mathcal{G}}}

\def\gL{{\mathcal{L}}}

\def\gN{{\mathcal{N}}}

\newcommand{\E}{\mathbb{E}}

\usepackage{cleveref}

\crefname{table}{Tab.}{Tabs.}
\Crefname{table}{Tab.}{Tabs.}
\crefname{equation}{Eq.}{Eqs.}
\Crefname{equation}{Eq.}{Eqs.}
\crefname{figure}{Fig.}{Figs.}
\Crefname{figure}{Fig.}{Figs.}
\crefname{section}{Sec.}{Secs.}
\Crefname{section}{Sec.}{Secs.}
\newtheorem{assumption}{Assumption}
\crefname{assumption}{assumption}{assumptions}
\newtheorem{proposition}{Proposition}
\crefname{proposition}{proposition}{propositions}
\newtheorem{theorem}{Theorem}
\crefname{theorem}{theorem}{theorems}
\newtheorem{lemma}{Lemma}
\newtheorem{corollary}{Corollary}
\crefname{corollary}{corollary}{corollaries}
\crefname{lemma}{lemma}{lemmas}

\def\Approach{DeDPO}

\newcommand{\ie}{i.e.,\ }

\usepackage{listings}

\author{%
  Khiem Pham\thanks{Equal contribution. Correspondence: \texttt{\{dkp45,tn375\}@cornell.edu} and \texttt{quang.ng@rutgers.edu}.} \\ %
  Cornell University
  \And
  Quang Nguyen\footnotemark[1] \\
  Rutgers University
  \And
  Tung Nguyen\footnotemark[1] \\
  Cornell Tech
  \And
  Jingsen Zhu \\
  Cornell University
  \And
  Michele Santacatterina \\
  New York University
  \And
  Dimitris Metaxas \\
  Rutgers University
  \And
  Ramin Zabih \\
  Cornell Tech
}

\begin{document}

\title{Debiased DPO for Diffusion Models}

\maketitle

\begin{abstract}
  Direct Preference Optimization (DPO) enables offline alignment of diffusion models from preference pairs without explicit reward modeling, but scaling DPO requires large amounts of expensive human preference data. In this paper, we consider a semi-supervised setting in which only a small subset of preference pairs is annotated by humans, while the remaining, much larger pool is labeled with inexpensive synthetic feedback, e.g., vision-language models scoring image pairs or self-training from the current model. However, such synthetic supervision is not a substitute for human judgment: it can be systematically misaligned, so naively mixing human and synthetic preferences can lead to misaligned model behavior. We introduce \Approach, a debiased objective that integrates a doubly robust estimator from causal inference into DPO to correct synthetic-label misalignment while preserving the simplicity of DPO. Experiments show improved label efficiency and robustness to synthetic-labeler quality, allowing \Approach~to outperform standard DPO under the same human-label budget and to approach fully human-supervised performance despite using four times fewer human-labeled examples.
\end{abstract}

\section{Introduction}
\label{sec:intro}

Text-to-image diffusion models have demonstrated impressive capabilities in generating images from natural language prompts~\cite{ho2023diffusion, dhariwal2021diffusion, rombach2022high}. Trained at scale on large image--text corpora~\cite{schuhmann2022laion}, these models generalize broadly, but their pretraining objectives do not directly capture subjective, user-valued criteria such as fine-grained prompt adherence, aesthetic preference, or style consistency. As a result, generated images can be plausible yet misaligned with what humans actually prefer.

\begin{figure}[ht]
  \centering
  \includegraphics[width=1.\linewidth]{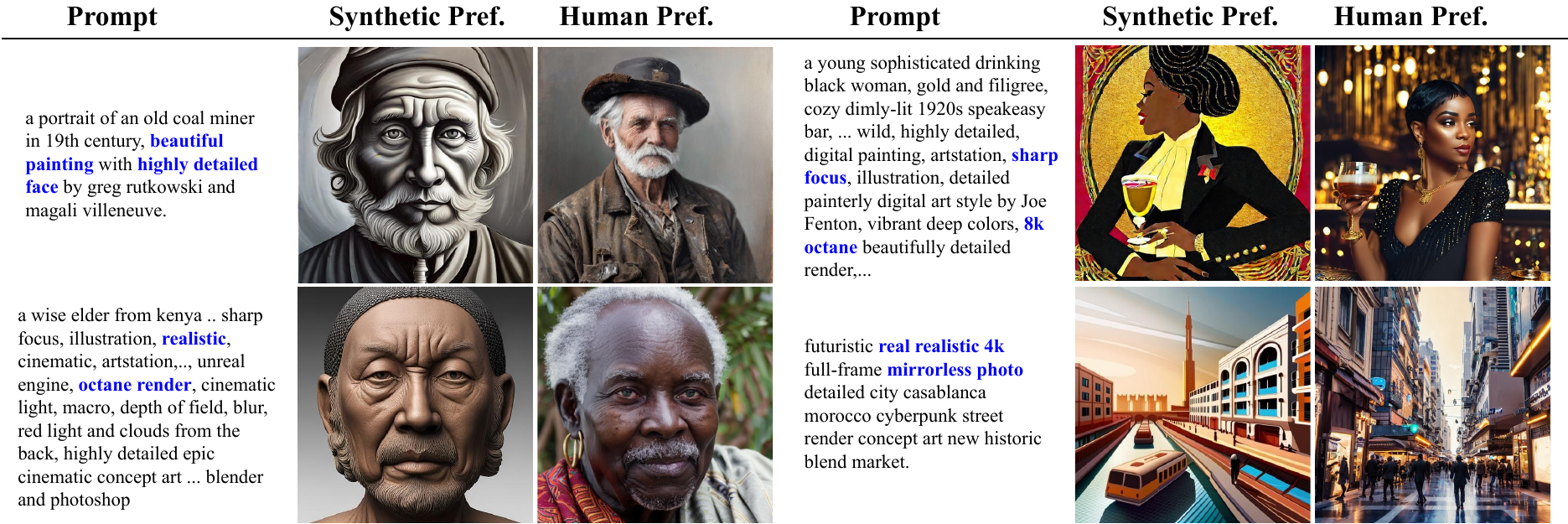}
  \caption{\textbf{Mismatch between synthetic and human preferences on Pick-a-Pic v2 \cite{Kirstain2023PickaPicAO}.} Qwen2.5-VL tends to favor semantic coherence and style consistency, but often underweights photorealism-related cues such as \textit{"8k"} or \textit{"hyper-detailed"}. Human raters place more weight on photorealistic fidelity, even when some style details are missed. This gap motivates DeDPO, which uses a small set of human labels to correct synthetic preference bias.}
  \label{fig:bias_analysis}
\end{figure}

Preference-based alignment addresses this gap by fine-tuning models using pairwise comparisons. In language models, RLHF~\cite{ouyang2022training,jaech2024openai} learns a reward model from human preferences and optimizes the policy via reinforcement learning~\cite{mnih2016asynchronous,schulman2017proximal}. Similar ideas have been applied to image generation~\cite{black2023training,fan2023dpok}. However, training and maintaining a reward model, as well as performing RL optimization, increases both system complexity and computational cost. Direct Preference Optimization (DPO)~\cite{rafailov2023direct} simplifies this pipeline by directly optimizing a preference objective from paired comparisons without explicitly learning a reward model. Diffusion-DPO~\cite{wallace2024diffusion} extends DPO to diffusion models and has been adopted in a growing set of generative tasks, including image inpainting~\cite{shen2025follow}, video generation~\cite{liu2025improving,wu2025densedpo}, and physically grounded 3D generation~\cite{li2025dso}. Its simplicity and offline nature make DPO-style training an appealing default for diffusion alignment.

Despite this appeal, \textbf{DPO-based methods inherit a key bottleneck: collecting large-scale, high-quality human preference labels is expensive and time-consuming}. This limitation motivates the use of \emph{synthetic} feedback, such as pairwise preferences generated by vision-language models, reward models, or self-training, to reduce reliance on costly human annotation~\cite{guo2024direct,Lee2024RLAIF}. However, synthetic preferences are not a drop-in replacement for human judgments because they can be systematically misaligned. As \cref{fig:bias_analysis} illustrates, this misalignment can vary across prompts, domains, and sampling settings. In practice, synthetic labels are often unavoidable when aligning large diffusion models under a limited human-label budget. The key question is therefore not whether to combine human and synthetic preferences, but how to combine them. Treating synthetic labels as ground truth and training DPO on both human and synthetic pairs can degrade alignment when the synthetic signal diverges from human preferences.

To address this problem, we propose \textbf{\Approach}, a debiased semi-supervised DPO objective for diffusion alignment. The main idea is simple: use synthetic preferences for coverage, but correct their systematic errors with a small set of high-quality human labels instead of treating them as ground truth. Concretely, \Approach~augments standard DPO with a doubly robust correction term while preserving DPO's practical advantages: fully offline training, no separate reward model, and compatibility with existing diffusion pipelines. Our contributions are threefold. \textbf{First}, we formulate diffusion alignment with mixed human and synthetic preferences as a debiased semi-supervised learning problem and derive a simple objective that corrects synthetic-label bias using labeled data. \textbf{Second}, building on the doubly robust semi-supervised estimator of~\cite{zhu2023doubly} and the orthogonal statistical learning framework of~\cite{foster2023orthogonal}, we instantiate this theory for Diffusion-DPO: we show that the resulting estimator is unbiased and doubly robust under standard assumptions, which helps explain why naive mixing can fail when synthetic preferences are misaligned. We do not claim a new general doubly robust estimator.
\textbf{Third}, across two datasets, different backbones, and several synthetic labelers, we show that \Approach~improves over matched DPO baselines, with the largest gains when synthetic labels are more misaligned with humans, and approaches the fully human-supervised reference point while using only \textbf{25\% of the human annotations}.
Throughout the paper, we use \emph{misalignment} and \emph{bias} interchangeably, except when referring to human bias. Here, statistical bias means a systematic divergence between a predictor's distribution and a target distribution, rather than simple random noise. In our setting, this includes both a model--human gap (e.g., Qwen2.5-VL preferring aesthetics in \cref{fig:bias_analysis}) and a human--prompt gap (e.g., humans favoring photorealism over strict prompt compliance). We use \emph{alignment} specifically to refer to reducing the model--human gap.

\section{Related Work}
\label{sec:related-work}
\myheading{Alignment with synthetic preferences.}
To reduce the cost of human comparisons, many diffusion-alignment methods rely on \emph{synthetic} preference labels. A common approach trains reward models such as PickScore and HPSv2 on large human-preference datasets and then uses them to annotate image pairs at scale~\cite{Kirstain2023PickaPicAO,wu2023human,wallace2024diffusion,karthik2025scalable}. Another line of work uses self-training or pseudo-labeling from the current diffusion model itself~\cite{lee2025semi, konwer2025enhancing, liulearning}. Both settings can suffer from \emph{double-dipping}, where the same model or human-preference source is reused for both labeling and training. We instead study a label-scarce regime with a small human-labeled set and a larger unlabeled pool labeled by pretrained external VLMs such as Qwen2.5-VL and SmolVLM~\cite{qwen2.5-VL,smolvlm}. Because these VLM judges are imperfect and systematically biased relative to humans, naively mixing their labels with scarce human data can misguide DPO. Our method addresses this with a principled debiasing mechanism anchored by the small human-labeled subset.

\myheading{Robust alignment with noisy labels.}
Synthetic preferences can disagree with human preferences, so robustness to label error is important. Existing work mainly follows two directions: methods based on explicit noise models~\cite{bukharin2024robust, chowdhury2024provably,mitchell2023cdpo} and methods based on distributionally robust optimization (DRO)~\cite{xu2025robust,wu2024towards,kim2025lightweight,mandal2025distributionally}. However, these approaches are primarily designed for random or worst-case label corruption. Our setting is different: the main problem is \emph{systematic bias} in synthetic preferences, and we also have access to a small set of high-quality human labels. This reference set lets us \emph{debias} synthetic supervision directly, rather than only smoothing labels or optimizing for worst-case robustness. Unlike robust reward-modeling approaches in RLHF~\cite{yan2024reward,chakraborty2024maxmin,fu2025reward,ye2025robust,padmakumar2024beyond}, our method stays reward-model-free and focuses on biased synthetic preferences in diffusion alignment.

\myheading{Doubly Robust / Debiased estimators.}
Doubly robust (DR) and debiased estimators were developed to handle biased or missing supervision~\cite{Robins1994,Scharfstein1999,Bang2005}. Their key idea is to combine two estimators of the same target so that the final estimator remains reliable if either one is accurate~\cite{chernozhukov2018double,kennedy2024semiparametric,foster2023orthogonal}. This makes them well suited to our setting, where synthetic preferences provide broad but imperfect supervision and a small human-labeled set provides a high-quality reference. Closest prior work applies DR ideas to semi-supervised classification~\cite{zhu2023doubly}, while recent alignment work such as Xu et al.~\cite{Xu2025} studies related debiasing ideas for language models. \textbf{Our setting differs in three ways}: we study \emph{pairwise preference optimization} rather than classification, we focus on diffusion alignment rather than language generation, and we use \emph{external} synthetic labelers instead of relying only on self-training. Compared with self-training methods~\cite{lee2025semi, konwer2025enhancing, liulearning}, external supervision avoids confirmation bias and better matches the sample-splitting conditions behind DR theory. Compared with Xu et al.~\cite{Xu2025}, our method remains fully offline and integrates debiasing directly into Diffusion DPO.

\section{Preliminaries}
\myheading{Text-to-image diffusion model.}
Diffusion models~\cite{song2020score, ho2020denoising, song2020denoising, dhariwal2021diffusion, rombach2022high} generate images by learning to reverse a gradual noising process applied to image latents. Formally, they model the transformation from a standard Gaussian noise map \(\boldsymbol{\veps} \sim \gN(0, \mI)\) to a sample from the target data distribution via a denoising diffusion process. The forward (noising) process perturbs a clean image latent \(\vx_0\) across multiple timesteps \(t\) as:
\begin{equation}
  \vx_t = \alpha_t \vx_0 + \sigma_t \veps,
  \label{eq:noise_t}
\end{equation}
where \(\alpha_t\) and \(\sigma_t\) are time-dependent scaling coefficients satisfying \(\alpha_0 = \sigma_T = 1\) and \(\alpha_T = \sigma_0 = 0\). The reverse process learns to recover \(\vx_0\) by denoising \(\vx_t\) using a noise estimator \(\veps_\theta\), trained to predict the noise \(\veps\) via:
\begin{equation}
  \min_{\theta} \mathbb{E}_{t \sim \mathcal{U}(0, T), \veps \sim \gN(0, \mI)} \left\| \veps_\theta(\vx_t, t) - \veps \right\|^2_2.
\end{equation}

In the text-to-image setting, models are further conditioned on text embeddings \(\rvc\), modifying the noise predictor to \(\veps_\theta(\vx_t, t, \rvc)\). This conditioning enables the generation to be guided by textual prompts. To strengthen alignment between the image and the text, many models incorporate classifier-free guidance (CFG)~\cite{dhariwal2021diffusion, ho2022classifier}, which blends conditional and unconditional predictions using a guidance scale \(\gamma\). The guided output is computed as:
\begin{equation}
  \hat{\veps}_\theta(\vx_t, t, \rvc) = \veps_\theta(\vx_t, t) + \gamma \left( \veps_\theta(\vx_t, t, \rvc) - \veps_\theta(\vx_t, t) \right),
\end{equation}
where \(\veps_\theta(\vx_t, t, \rvc)\) denotes the prediction conditioned on the text, and \(\veps_\theta(\vx_t, t)\) denotes the unconditional prediction. CFG enhances both the fidelity and semantic relevance of generated images without relying on an external classifier.

\myheading{DPO for text-to-image diffusion models.} Direct Preference Optimization (DPO)~\cite{rafailov2023direct} offers an efficient alternative to reward modeling for aligning generative models with human preferences. Diffusion-DPO~\cite{wallace2024diffusion} extends this approach to text-to-image diffusion models by adapting the DPO objective to the denoising process. Since diffusion models generate images through a sequence of latent states over \(T\) timesteps, directly computing the full trajectory likelihood is intractable. To address this, Diffusion-DPO proposes a surrogate loss based on denoising score comparisons at individual timesteps. For a given text prompt \(\vc\) and an image pair \((\vx_t^w, \vx_t^l)\) annotated by human preference $\vx_t^w\succ\vx_t^l$, with Gaussian noise samples \((\veps^w, \veps^l) \sim \gN(0, \mI)\), the preference loss is approximated as:
\begin{equation}
  \label{eq:dpo}
  \begin{aligned}
    L_{\text{DPO}}(\theta)
    &= \E_{t, \vc, \vx_t^w, \vx_t^l}
    \left[-\log \gG_\theta(\vc, \vx_t^l, \vx_t^w)\right], \\
    \gG_\theta(\vc, \vx_t^l, \vx_t^w)
    &= \sigma\Bigl(
      -\beta \bigl(
        \|\veps^w - \veps_\theta(\vx_t^w, t,c)\|_2^2
        - \|\veps^w - \veps_\tref(\vx_t^w, t,c)\|_2^2 \\
        &\qquad\qquad\;
        - \|\veps^l - \veps_\theta(\vx_t^l, t,c)\|_2^2
        + \|\veps^l - \veps_\tref(\vx_t^l, t,c)\|_2^2
      \bigr)
    \Bigr).
  \end{aligned}
\end{equation}
where $\sigma(\cdot)$ is the sigmoid function, $\vx_t^w$ and $\vx_t^l$ are the noise-perturbed images at timestep $t$ (\cref{eq:noise_t}), \(\veps_\theta\) is the trainable noise predictor and \(\veps_{\text{ref}}\) is a frozen reference model. This formulation encourages the model to improve denoising quality on preferred samples relative to less preferred ones, effectively aligning generation with preference signals without relying on reward modeling.

\section{Method}
\label{sec:method}
\subsection{Problem setting.}
We consider a semi-supervised setting for text-to-image diffusion alignment via DPO, where only a small set of prompts is annotated with human preference pairs alongside a larger pool of synthetically labeled data. Our objective is to leverage limited human supervision together with abundant but noisy synthetic labels to achieve performance competitive with fully human-annotated training.

\myheading{DPO as binary classification.} We first re-interpret DPO in \cref{eq:dpo} as an equivalent binary classification problem. For an \emph{unordered} image pair \((\vx^0, \vx^1)\) with text prompt $\vc$, we denote $\vz$ as the preference label over the pair, where $\vz=1$ denotes preference on $\vx^1$ and $\vz=0$ otherwise. With this label, we can rewrite the loss in \cref{eq:dpo} as a binary cross-entropy loss
\begin{equation}
  \label{eq:dpo-class}
  \begin{aligned}
    L_{\text{DPO}}(\theta)
    &= \E_{t, \vc, \vx_t^0, \vx_t^1}
    \left[\gL\bigl(\gG_\theta(\vc, \vx_t^0, \vx_t^1), \vz\bigr)\right], \\
    \gL(a, b)
    &= -b \log a - (1-b) \log(1-a).
  \end{aligned}
\end{equation}
which is a binary classification loss with predictive model $\hat\vz = \gG_\theta(\vc,\vx_t^0,\vx_t^1)$ and target label $\vz$.

\myheading{Notes on symbols.}
We drop the subscript $t$ for brevity. For example, $\vx_t^0$ is written as $\vx^0$. We also drop the subscripts of the expectation operator $\E$. $\E_n$ denotes the average over $n$ samples (equivalent to the operator $\frac{1}{n} \sum_{i=1}^n \cdot$). $\vy$ denotes the input triplet $\vy = (\vc, \vx^0, \vx^1)$. Therefore, \cref{eq:dpo-class} can be written as:
\begin{equation}
  \label{eq:short-dpo}
  L_{\text{DPO}}(\theta) = \E_n \gL (\gG_\theta(\vy), \vz)
\end{equation}
We will slightly overload $\gG$ to mean any mapping from the input $\vy$ to the label prediction $\hat \vz$, reserving $\gG_\theta$ for the model's implicit margin (\cref{eq:dpo}). Additionally, we note that the expectation is also taken over the denoising time $t$.

\myheading{Naive pseudo-labeling is biased.}
In our semi-supervised setting, let $\mathcal{D}_l=\{(\vy_i,z_i)\}_{i=1}^{n_l}$ denote labeled pairs with human preferences,
and $\mathcal{D}_u=\{\vy_j\}_{j=1}^{n_u}$ denote unlabeled pairs.
For unlabeled data, we obtain synthetic preferences $\hat z=\hat{\gG}(\vy)$ from a cheap annotator $\hat{\gG}$ (e.g., a pretrained VLM).
A natural baseline is to \emph{impute} the missing human label on $\mathcal{D}_u$ with $\hat z$, and optimize the empirical risk:
\begin{equation}\label{eq:naive}
  L_{\text{OR}}(\theta)
  =\E_{n_l}\gL(\gG_\theta(\vy_l), z_l)
  +\E_{n_u}\gL(\gG_\theta(\vy_u), \hat{\gG}(\vy_u)).
\end{equation}
This objective corresponds to standard pseudo-label training: it treats $\hat z$ as if it were ground-truth supervision and simply mixes labeled and pseudo-labeled data. However, this approach is generally biased as an estimator of the true DPO risk. When $\hat{\gG}$ is biased, $\hat z$ deviates from the true human preference $z$ in a \emph{directional} way, not merely as random noise.
As a result, minimizing \cref{eq:naive} can drive $\theta$ toward the synthetic labeler's preferences rather than human preferences.

\subsection{DeDPO loss.} Our \textbf{key contribution} in the paper is to use debiased learning: use the unbiased signal from the small labeled set to alleviate bias from the large pseudo-labeled set. Concretely, the algorithm works by training on synthetic labels for \emph{all} samples, and adds a correction term on $\mathcal{D}_l$ that replaces pseudo labels by true labels:
\begin{equation}
  \label{eq:dr}
  \begin{aligned}
    L_{\text{\Approach}}(\theta)
    &= \E_{n_l+n_u}\Bigl[\gL\bigl(\gG_\theta(\vy), \hat{\gG}(\vy)\bigr)\Bigr] \\
    &\quad + \E_{n_l}\Bigl[
      \gL\bigl(\gG_\theta(\vy_l), z_l\bigr)
      - \gL\bigl(\gG_\theta(\vy_l), \hat{\gG}(\vy_l)\bigr)
    \Bigr].
  \end{aligned}
\end{equation}
In the following section, we demonstrate the properties that make DeDPO an effective loss function for alignment: (i) \cref{eq:dr} is \emph{unbiased} regardless of the synthetic label quality
(ii) the estimation error of the learned logit depends only \emph{weakly} on synthetic-label error.

\subsection{Properties of DeDPO.}
\subsubsection{DeDPO is unbiased}
A key property of \cref{eq:dr} is that the labeled-set correction term exactly cancels the bias introduced by training on synthetic labels.
As a result, \cref{eq:dr} is an \emph{unbiased} estimator of the population DPO risk, regardless of the quality of the synthetic labeler $\hat{\gG}$.

\begin{proposition} \label{prop:unbiasedness}
  The expectation of our proposed debiased loss in \cref{eq:dr} is equal to the expectation of the original DPO loss in \cref{eq:dpo}, \ie $L_\text{DeDPO}$ is unbiased:
  \begin{equation}
    \E [L_\text{\Approach}(\theta)] = \E[L_{\text{DPO}}(\theta)]
  \end{equation}
\end{proposition}
\Cref{prop:unbiasedness} guarantees that using large amounts of synthetic supervision does not introduce asymptotic bias.
As a sanity check, if $\hat{\gG}(\vy)=z$ is perfect, the correction term vanishes and \cref{eq:dr} reduces to training on all $n_l+n_u$ samples.
If $\hat{\gG}$ is low-quality, the correction term anchors learning to the human-labeled subset while still allowing the unlabeled pool to provide coverage.
The proof of \cref{prop:unbiasedness} is the same in \cite{zhu2023doubly} but it only shows the average effect of the debiased loss on the synthetic labels. Following, we explain its effect on individual labels. First, we combine the 3 loss terms of \Cref{eq:dr} into one with a new target label:
\begin{proposition}
  Our debiased loss is characterized by using pure pseudo-labels for unlabeled data and applying an amplified correction from pseudo-labels toward ground-truth labels for labeled data.
  \begin{equation}
    \begin{aligned}
      L_{\text{\Approach}}(\theta)
      &= \E_{n_l + n_u}\Bigl[\gL\bigl(\gG_\theta(\vy), \hat{\vw}\bigr)\Bigr], \\
      \hat{\vw}
      &=
      \begin{cases}
        \hat{\gG}(\vy), & \text{if } \vy \text{ is unlabeled}, \\
        \hat{\gG}(\vy) + \dfrac{n_l + n_u}{n_l}\bigl(\rvz - \hat{\gG}(\vy)\bigr), & \text{if } \vy \text{ is labeled with } \rvz.
      \end{cases}
    \end{aligned}
  \end{equation}
  \label{prop:geometric}
\end{proposition}

Please refer to the supplementary for the proof. Intuitively, for unlabeled data points, our loss simply uses the pseudo-labels as the supervision signal. For labeled data points, the loss introduces a correction term $\rvz - \hat \gG(\vy)$ that moves the pseudo-label toward the true label, amplified by the importance weight $\frac{n_l + n_u}{n_l} \ge 1$, which reflects the local density ratio. When the pseudo-label is already correct, this correction term vanishes and the loss reduces to standard supervised learning. When the pseudo-label is incorrect, the amplified correction strengthens the influence of the human label on training. Through shared model parameters, corrections on labeled examples can also influence predictions on unlabeled examples.

\subsubsection{Convergence of DeDPO}
\label{sec:main-convergence}
We now discuss a second theoretical property of our debiased loss,
namely how the convergence rate depends on the number of samples and
the quality of the synthetic preference model $\hat{\gG}$.
Let $f_\theta(\vy)=\operatorname{logit}\gG_\theta(\vy)$,
$\hat f=f_{\hat\theta}$, and let $f^*$ minimize the population DPO risk
over the logit class. Write $\gG^*(\vy)=\E[\vz\mid \vy]$.
Although we use a fixed labeling probability $p = \frac{n_l}{n_l + n_u}$, we state the theorem
more generally for probabilities that may depend on $\vy$ and require
estimation. With $\vr$ indicating an observed label, the inverse weight in
\cref{prop:geometric} becomes $\gamma^*(\vy)=1/\Pr(\vr=1\mid \vy)$, estimated by
$\hat\gamma(\vy)$. Our setting corresponds to $\hat\gamma=\gamma^*=1/p = \frac{n_l + n_u}{n_l}$.

\begin{theorem}[Informal]
  \label{thm:dr}
  Suppose the synthetic labeler $\hat{\gG}$ and any estimate
  $\hat\gamma$ are independent of the training sample. Under bounded
  logits and inverse weights, a logit class of finite VC-subgraph
  dimension, and the sampling and optimality conditions stated in the
  appendix, the empirical DeDPO minimizer satisfies, with high probability,
  \begin{equation}
    \label{eq:dr-convergence}
      \|\hat f-f^*\|_2^2
        \le \widetilde O\!\left(\frac{1}{n_l+n_u}\right)
        +O\!\left(
          \|\hat\gamma-\gamma^*\|_4^2
          \|\hat{\gG}-\gG^*\|_4^2
        \right).
  \end{equation}
\end{theorem}

The notation $\widetilde O$ suppresses logarithmic factors. The full
statement and proof follow the framework of \cite{foster2023orthogonal}
and are given in the appendix. The theorem has two useful consequences.

\begin{corollary}[Known or estimated propensity]
  \label{cor:propensity-cases}
  Under the assumptions of \cref{thm:dr}, with $n=n_l+n_u$,
  $\|\hat f-f^*\|_2^2\le\widetilde O(1/n)$ with high probability if either:
  \begin{enumerate}
    \item $\hat\gamma=\gamma^*$, with no accuracy requirement on
      $\hat{\gG}$; or
    \item with high probability,
      $\|\hat\gamma-\gamma^*\|_4\|\hat{\gG}-\gG^*\|_4=O(n^{-1/2})$,
      as when both nuisance errors are $O(n^{-1/4})$.
  \end{enumerate}
\end{corollary}

The first case applies directly to our random labeling design. The
nuisance term vanishes, so persistent synthetic-label bias does not
prevent convergence to the population optimum. No convergence of the
synthetic preference model is required.

For estimated propensity, fourth-order dependence is recovered when
the propensity error is bounded by a constant times the synthetic
preference error. The nuisance term is then
$O(\|\hat{\gG}-\gG^*\|_4^4)$, so the synthetic preference model only needs
to converge at the slower rate $n^{-1/4}$ for the learned logit to retain
the statistical rate $\widetilde O(1/n)$.

Notably, the independence assumption is violated by self-training on
the same preference labels but is satisfied by a fixed pretrained
external labeler. This ``double-dipping'' can have practical consequences:
if the synthetic preference model is fitted on the same labeled data
used for correction, the correction in \cref{prop:geometric} can overfit toward
zero and provide little corrective signal. \textit{Therefore, debiasing
may be less effective with self-training than with an external model};
see \cref{fig:labeler_sens_real} and the accompanying discussion.

\section{Experiments}
\subsection{Experimental setup}

\label{sec:exp-setup}
\myheading{Datasets.}
We use Pick-a-Pic v2~\cite{Kirstain2023PickaPicAO} and HPDv2~\cite{wu2023human} for training. For each dataset, we construct a high-quality 5K subset by ranking preference pairs with PickScore~\cite{Kirstain2023PickaPicAO} and retaining the top examples. This step removes many ambiguous or low-quality pairs, making the scarce human-labeled subset a cleaner and more label-efficient anchor for debiasing. To verify that this curation does not drive our conclusions, \cref{app:uncurated} repeats the main comparison on a uniformly sampled, unfiltered 5K subset. From this 5K pool, we sample a small human-labeled subset and a larger unlabeled subset, then replace the labels of the unlabeled portion with synthetic preferences. Unless otherwise stated, we use a 25\%/75\% split (1.25K human-labeled pairs and 3.75K pseudo-labeled pairs) and Qwen2.5-VL as the synthetic labeler, yielding the PaP-5K and HPD-5K settings reported in Table~\ref{tab:main}. We use PartiPrompt~\cite{yu2022parti} for evaluation, which contains 1632 test prompts.

\myheading{Notes on notations.} A two-number suffix denotes the human-labeled and unlabeled split (e.g., DeDPO+25-75 uses 25\% human labels and 75\% synthetic/unlabeled data), while a single number denotes fully supervised training (e.g., DPO+25 uses 25\% human labels only). Unless noted otherwise, DeDPO and DPO refer to the 25-75 setting.

\myheading{Metrics.} We evaluate preference alignment using three primary metrics: PickScore~\cite{Kirstain2023PickaPicAO}, HPSv3~\cite{ma2025hpsv3}, and ImageReward~\cite{xu2023imagereward}. For \textbf{PickScore} and \textbf{HPSv3}, we report the mean score over the 1632 PartiPrompt test prompts. For ImageReward, we report \textbf{ImageReward@Top-1(\%)}, defined as the percentage of prompts for which a model is ranked first among the compared methods.

\myheading{Baselines.}
We compare \textbf{\Approach} against three primary baselines, each targeting a specific question: the effect of preference optimization, synthetic supervision, and debiasing.
\textbf{SFT} is a non-preference-optimized reference for measuring the gain from DPO-style alignment over standard supervised fine-tuning.
\textbf{DPO+25} uses the same human-label budget as \Approach~but no synthetic labels, testing whether synthetic preferences add useful supervision.
\textbf{DPO+25-75} uses the same human/synthetic split as \Approach~but treats synthetic labels as ground truth, serving as the direct baseline for isolating the effect of our debiasing correction.
We additionally report \textbf{DPO+100} as a full-human-supervision reference point, not as a primary baseline \Approach~is expected to outperform, since it uses four times more human annotations than \textbf{DeDPO+25-75}.
It measures whether \Approach~matches DPO+100 on automated metrics while using substantially fewer human labels.
We evaluate all methods on \textbf{SD1.5} and \textbf{SDXL} using the metrics described above.

\myheading{Implementation details.} \textbf{(i) Training.} For SD1.5, we
fine-tune for 1{,}000 steps using AdamW~\cite{loshchilov2017AdamW} with a learning
rate of $1\times10^{-7}$, a piecewise-constant schedule with 10 warmup steps, and a
global batch size of 128 across 2 GPUs with 8 gradient accumulation steps. For SDXL,
we train for 100 steps using AdaFactor~\cite{shazeer2018adafactor} with a learning
rate of $2\times10^{-8}$ and 5 warmup steps under the same schedule. \textbf{(ii) Inference.} At test time, we generate images using 20 denoising
steps with DPMSolver++~\cite{lu2025dpm} and a classifier-free guidance scale of
7.5, applied uniformly across all methods and model scales. All results in \cref{tab:main} are averaged over 15 random seeds. Each training mini-batch contains equal numbers of labeled and unlabeled pairs, which sets the correction weight to $w_t=2$ (\cref{app:minibatch}).

\subsection{Experimental results}
\begin{table}[t]

  \scriptsize
  \caption{Main results on SD1.5 and SDXL trained on PaP-5K and HPD-5K, derived from PickaPicv2~\cite{yang2024automated} and HPDv2~\cite{wu2023human} (Section~\ref{sec:exp-setup}). \#L / \#U: number of labeled / unlabeled pairs; \#U=0 denotes full supervision. PickScore and HPSv3 are reported as mean $\pm$ std over 15 seeds.}
  \label{tab:main}
  \setlength{\tabcolsep}{4pt}
  \centering
  \begin{tabular}{lllcccccc}
    \toprule
    \bf Train set & \bf Model & \bf Method & \bf \#L & \bf \#U & \bf Pickscore \cite{Kirstain2023PickaPicAO} (\(\uparrow\)) & \bf HPSv3 \cite{ma2025hpsv3} (\(\uparrow\)) & \bf ImageReward@Top-1(\%) \cite{xu2023imagereward} (\(\uparrow\)) \\
    \midrule
    \multirow{10}{*}{PaP-5K}
    & \multirow{5}{*}{SD1.5}
    & SFT &  1250 &  0 &21.64 \(\pm\) 0.011 & 4.81 \(\pm\) 0.074 & 15.32 \\
    & & DPO+25 & 1250 & 0 & 21.76 \(\pm\) 0.011 & 4.99 \(\pm\) 0.073 & 15.99 \\
    & & DPO+25-75 &1250 & 3750 & 21.71 \(\pm\) 0.017 & 5.07 \(\pm\) 0.074 & 19.42 \\
    & & \Approach+25-75 & 1250 & 3750 & \textbf{21.91 \(\pm\) 0.011} & \textbf{5.19 \(\pm\) 0.071} & \textbf{22.37} \\
    & & \textcolor{gray!80}{DPO+100} & \textcolor{gray!80}{5000} & \textcolor{gray!80}{0} & \textcolor{gray!80}{21.88 \(\pm\) 0.012} & \textcolor{gray!80}{5.15 \(\pm\) 0.072} & \textcolor{gray!80}{26.90} \\
    \cmidrule(lr){2-8}
    & \multirow{5}{*}{SDXL}
    & SFT & 1250 & 0 & 22.01 \(\pm\) 0.013 & 5.24 \(\pm\) 0.079 & 17.59 \\
    & & DPO+25 & 1250 & 0 & 22.57 \(\pm\) 0.016 & 6.53 \(\pm\) 0.066 & 18.69 \\
    & & DPO+25-75 & 1250 & 3750 & 22.61 \(\pm\) 0.014 & 6.71 \(\pm\) 0.061 & 19.61 \\
    & & \Approach+25-75 & 1250 & 3750 & \textbf{22.83 \(\pm\) 0.013} & \textbf{6.99 \(\pm\) 0.070} & \textbf{21.94} \\

    & & \textcolor{gray!80}{DPO+100} & \textcolor{gray!80}{5000} & \textcolor{gray!80}{0} & \textcolor{gray!80}{22.84 \(\pm\) 0.015} & \textcolor{gray!80}{7.01 \(\pm\) 0.069} & \textcolor{gray!80}{22.18} \\
    \midrule
    \multirow{10}{*}{HPD-5K}
    & \multirow{5}{*}{SD1.5}
    & SFT & 1250 & 0 & 21.48 \(\pm\) 0.011 & 4.65 \(\pm\) 0.074 & 15.81 \\
    & & DPO+25 & 1250 & 0 & 21.61 \(\pm\) 0.011 & 4.75 \(\pm\) 0.073 & 18.22 \\
    & & DPO+25-75 & 1250 & 3750 & 21.69 \(\pm\) 0.017 & 4.76 \(\pm\) 0.074 & 20.21 \\
    & & \Approach+25-75 & 1250 & 3750 & \textbf{21.71 \(\pm\) 0.011} & \textbf{4.79 \(\pm\) 0.071} & \textbf{22.26} \\

    & & \textcolor{gray!80}{DPO+100} & \textcolor{gray!80}{5000} & \textcolor{gray!80}{0} & \textcolor{gray!80}{21.61 \(\pm\) 0.012} & \textcolor{gray!80}{4.79 \(\pm\) 0.072} & \textcolor{gray!80}{22.50} \\
    \cmidrule(lr){2-8}
    & \multirow{5}{*}{SDXL}
    & SFT & 1250 & 0 & 21.60 \(\pm\) 0.013 & 5.11 \(\pm\) 0.079 & 15.83 \\
    & & DPO+25 & 1250 & 0 & 22.48 \(\pm\) 0.016 & 6.28 \(\pm\) 0.066 & 17.82 \\
    & & DPO+25-75 & 1250 & 3750 & 22.52 \(\pm\) 0.014 & 6.38 \(\pm\) 0.062 & 20.12 \\
    & & \Approach+25-75 & 1250 & 3750 & \textbf{22.55 \(\pm\) 0.013} & \textbf{6.46 \(\pm\) 0.070} & \textbf{22.81} \\

    & & \textcolor{gray!80}{DPO+100} & \textcolor{gray!80}{5000} & \textcolor{gray!80}{0} & \textcolor{gray!80}{22.53 \(\pm\) 0.015} & \textcolor{gray!80}{6.49 \(\pm\) 0.069} & \textcolor{gray!80}{23.42} \\
    \bottomrule
  \end{tabular}
\end{table}

\myheading{Comparison to DPO \cite{wallace2024diffusion}.}
\cref{tab:main} shows that \Approach~improves over the two direct baselines, DPO+25 and DPO+25-75, across both datasets and both backbones. These are the key comparisons because DPO+25 matches the human-label budget, while DPO+25-75 matches the same human/synthetic split without debiasing. The clearest gains appear on PaP-5K, where synthetic labels are less aligned with human preferences. Relative to DPO+25-75, \Approach~improves PickScore by \(+0.20/+0.22\), HPSv3 by \(+0.12/+0.28\), and ImageReward@Top-1 by \(+2.95/+2.33\) points for SD1.5/SDXL. These gains are larger than the corresponding run-to-run variation shown in Table~\ref{tab:main}. It also exceeds DPO+25 in every PaP-5K setting, suggesting that synthetic preferences are most useful when combined with explicit debiasing rather than naively treated as ground truth. On HPD-5K, the gains over DPO+25-75 are smaller: \(+0.02/+0.03\) in PickScore, \(+0.03/+0.08\) in HPSv3, and \(+2.05/+2.69\) points in ImageReward@Top-1 for SD1.5/SDXL. For PickScore and HPSv3, these smaller differences should be interpreted as matching within error bars or improving within seed variance rather than as clear statistical separations. This smaller margin is expected because synthetic labels on HPD-5K agree more closely with human preferences, about 93\% versus 83\% on PaP-5K, leaving less bias for the correction term to remove. Compared with SFT, all preference-optimized methods are substantially stronger on automated metrics, confirming that preference optimization is itself important for alignment. \Approach~also matches DPO+100 on some automated metrics and remains close on others while using only 25\% of the human supervision. Measured as the fraction of the DPO+25-75$\rightarrow$DPO+100 gap that is closed, \Approach~recovers 96--100\% of the PickScore gap on SD1.5 and SDXL (\cref{app:gap}). The advantage over DPO+25-75 also persists without PickScore-based curation (\cref{app:uncurated}), on a DiT backbone (PixArt-$\Sigma$, \cref{app:dit}), and at a 50/50 human/synthetic split (\cref{app:5050}).

\begin{figure}[t]
  \centering

  \begin{subfigure}[t]{0.32\linewidth}
    \centering
    \includegraphics[height=3.5cm, width=\linewidth]{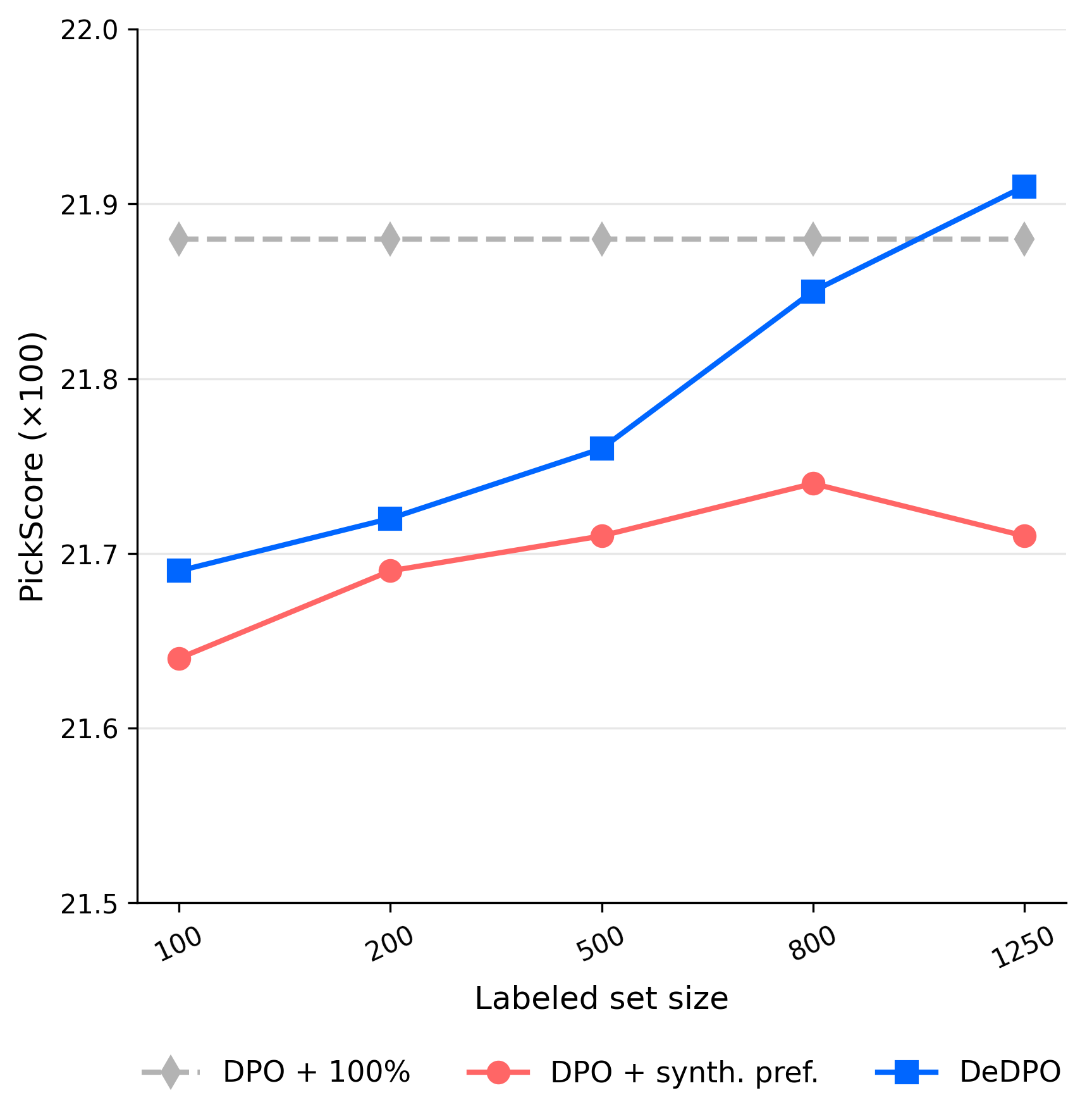}
    \caption{PickScore under sparse hu-
      man supervision: fixed 5K total datapoints with
    varying labeled set size. Gray dashed line shows fully supervised DPO on 5K human labels.}
    \label{fig:sparse_regime}
  \end{subfigure}\hfill
  \begin{subfigure}[t]{0.32\linewidth}
    \centering
    \includegraphics[height=3.5cm, width=\linewidth]{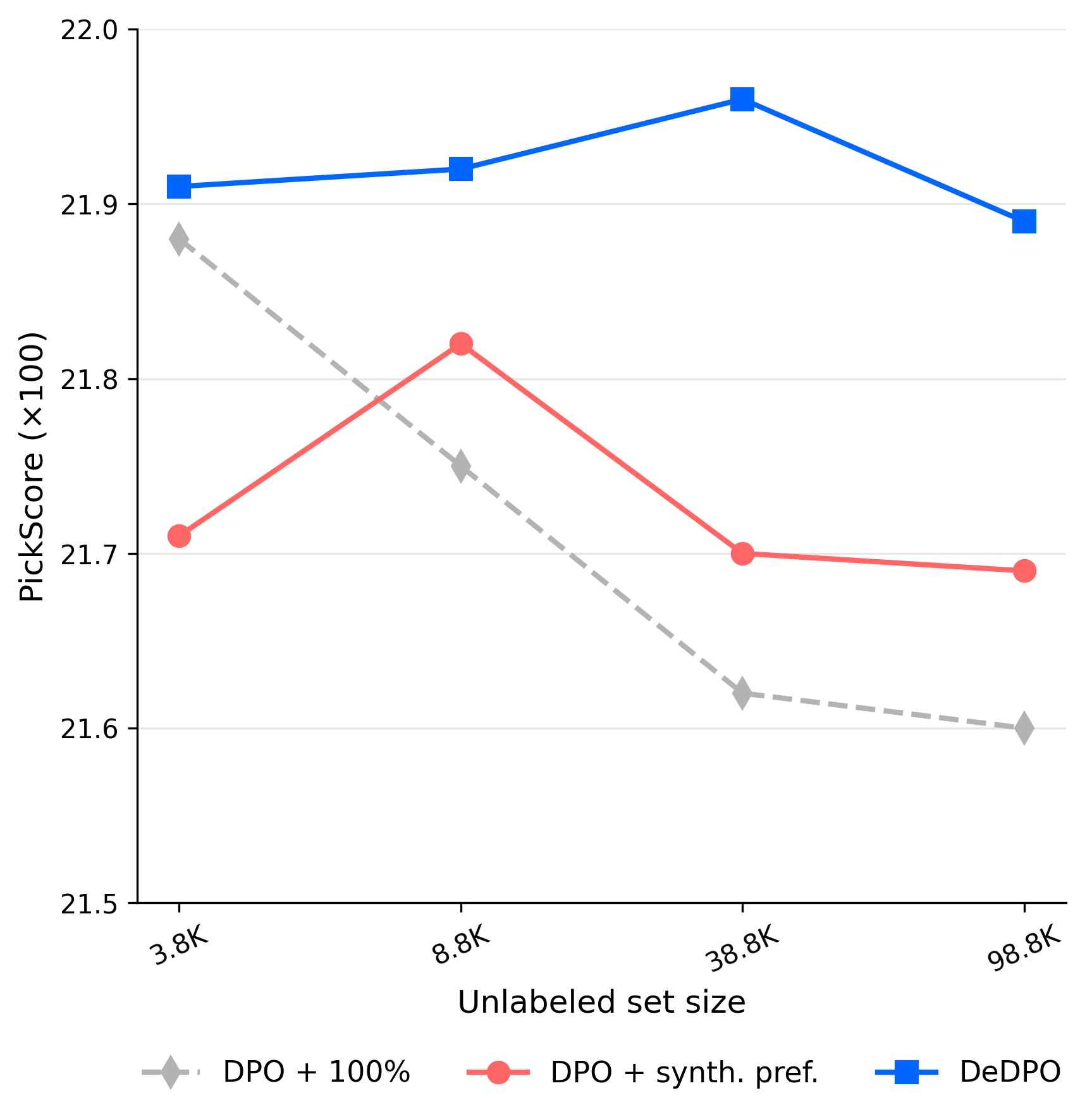}
    \caption{Effect of scaling synthetic preferences with fixed 1.25K human labels. Gray dashed line shows fully supervised DPO on 5K–100K human labels.}
    \label{fig:scaling_synth}
  \end{subfigure}\hfill
  \begin{subfigure}[t]{0.32\linewidth}
    \centering
    \includegraphics[height=3.5cm, width=\linewidth]{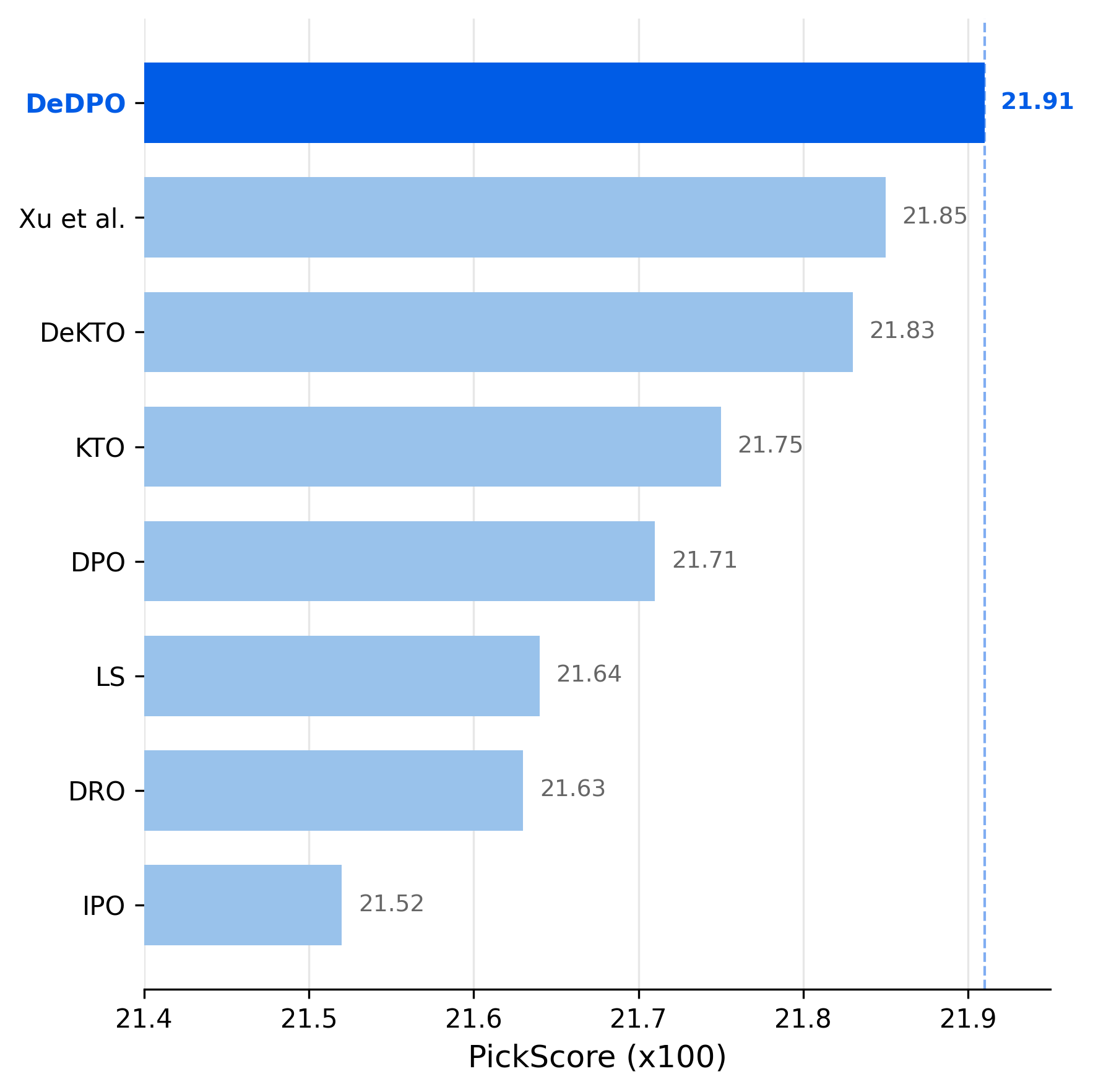}
    \caption{Comparison of DeDPO with other robust methods. 1.25K human labels and 3.75K Qwen2.5-VL labels.}
    \label{fig:other_models}
  \end{subfigure}
  \caption{Performance under label-scarce regimes. \textbf{(a)} \Approach~is numerically above DPO across the plotted human-label budgets (100–1250 pairs). \textbf{(b)} \Approach~shows a smaller drop as the synthetic pool scales, with the gray dashed line showing fully supervised DPO on 5K–100K PaP human labels as a reference. \textbf{(c)} \Approach~exceeds several compared baselines in this setting.}
  \label{fig:sparse_scale_drop}

\end{figure}

\myheading{Comparison to other robust methods.}
We compare DeDPO against LS~\cite{mitchell2023cdpo}, IPO~\cite{azar2024general}, DRO~\cite{wu2024towards}, Xu et al.~\cite{Xu2025}, and KTO~\cite{li2024aligning}. LS and DRO address noisy preferences through label smoothing and worst-case distributional optimization, respectively. In our setting, however, the main issue is \emph{systematic synthetic bias}: the model labeler makes errors in a structured, non-random way. As a result, LS and DRO are mismatched to this problem and underperform DeDPO in Fig.~\ref{fig:other_models}. We also test DeKTO, a debiased variant of the point-wise KTO loss, and find that it improves over KTO numerically (21.83 vs.\ 21.75), suggesting that the same debiasing idea may extend beyond DPO-style objectives. Xu et al. is the numerically closest competitor (21.85 vs.\ 21.91), but requires more than 2$\times$ the training time due to online generation. Overall, \cref{fig:other_models} shows that DeDPO achieves the strongest reported PickScore among the compared methods in this setting.

\begin{figure}[ht]
  \centering

  \begin{subfigure}[t]{0.32\linewidth}
    \centering
    \includegraphics[height=3.5cm, width=\linewidth]{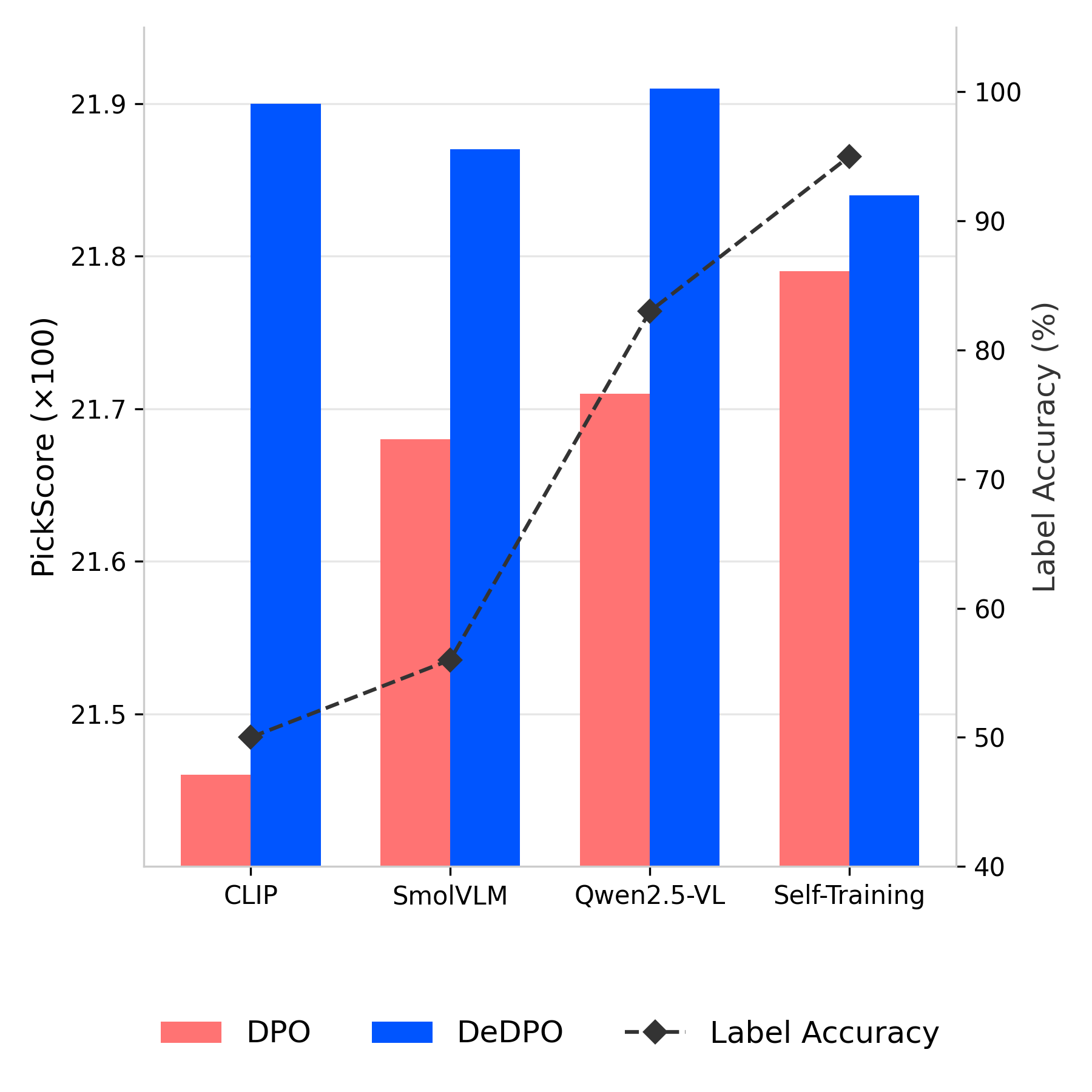}
    \caption{VLM Labeler}
    \label{fig:labeler_sens_real}
  \end{subfigure}\hfill
  \begin{subfigure}[t]{0.32\linewidth}
    \centering
    \includegraphics[height=3.5cm,width=\linewidth]{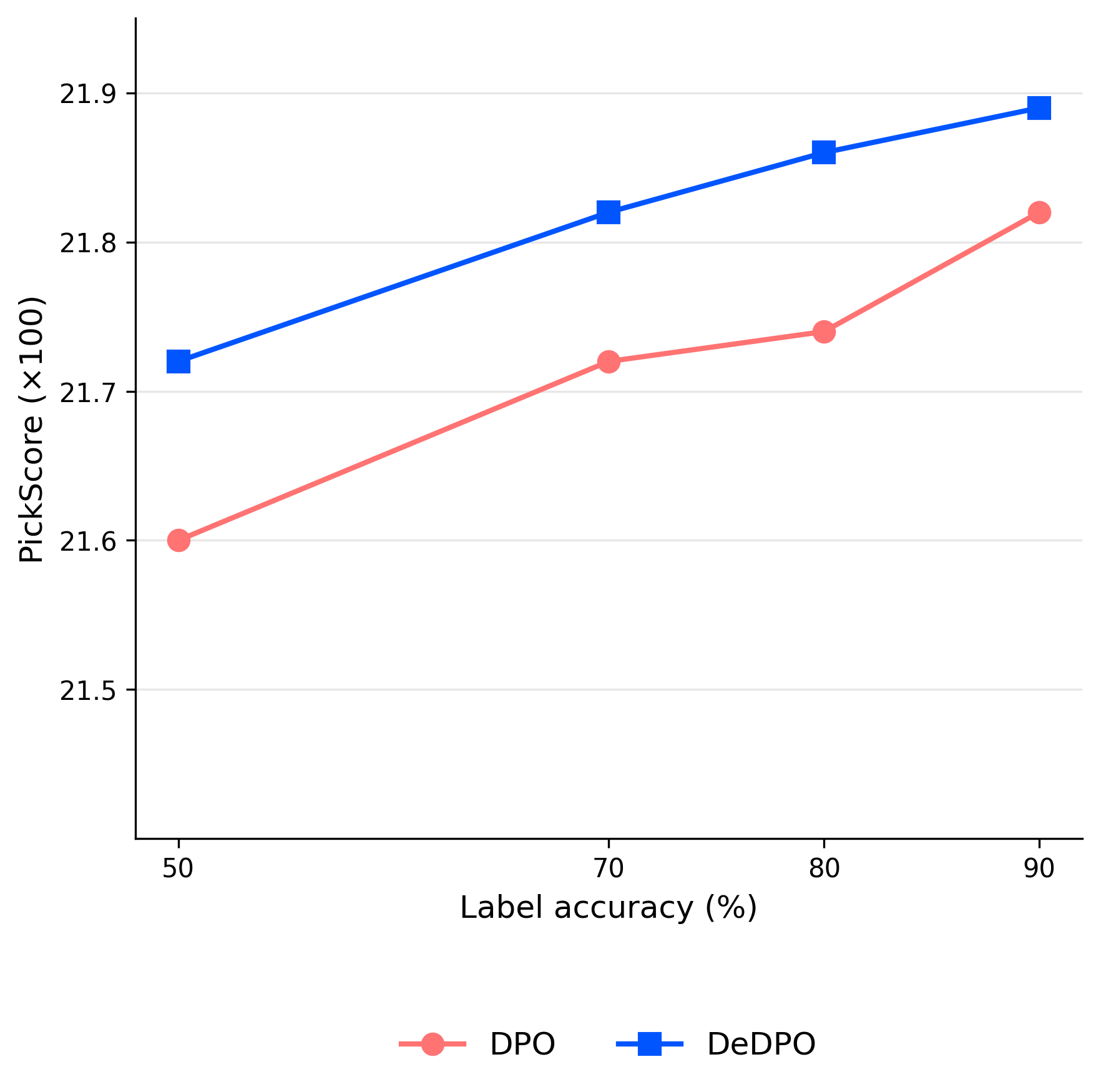}
    \caption{Systematic rater bias}
    \label{fig:labeler_sens_synbias}
  \end{subfigure}\hfill
  \begin{subfigure}[t]{0.32\linewidth}
    \centering
    \includegraphics[height=3.5cm, width=\linewidth]{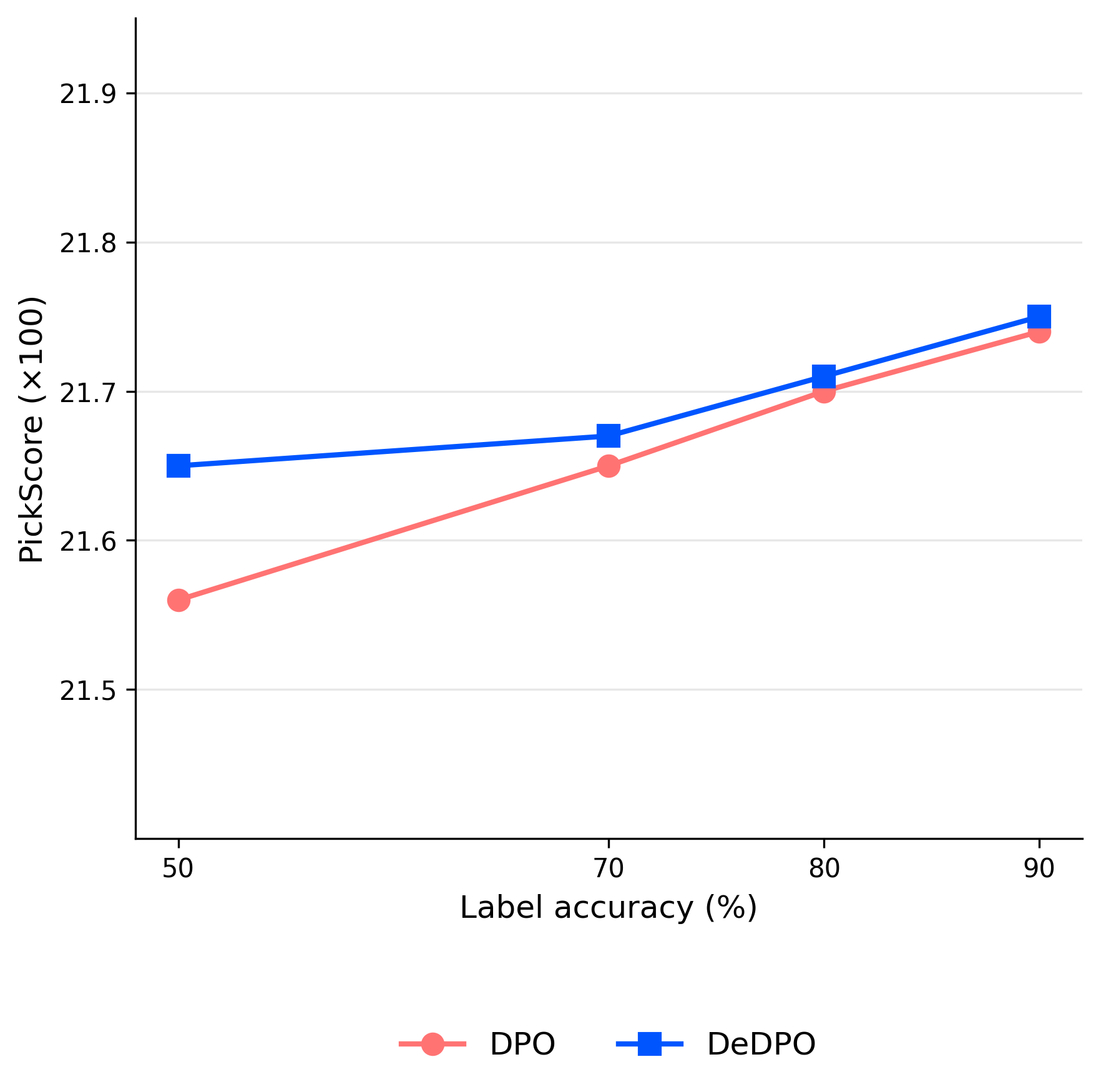}
    \caption{Random Preference Flip}
    \label{fig:labeler_sens_flip}
  \end{subfigure}
  \caption{\textbf{Labeler sensitivity under three types of corruption.}
    (PickScore $\times 100$, SD1.5, PaP-5K). Label accuracy denotes the
    percentage of uncorrupted synthetic labels. \textcolor{blue}{DeDPO+25-75} is numerically above \textcolor{red}{DPO+25-75} across
    the plotted labeler types and corruption levels, with the largest visible gaps under
  structured, directional bias.}
  \label{fig:labeler_sensitivity}
\end{figure}

\myheading{Human evaluation.}

We conduct a blind pairwise human evaluation on HPSv2 prompts~\cite{wu2023human} from four categories: \textit{anime}, \textit{photo}, \textit{paintings}, and \textit{concept-arts}. For each comparison, approximately 20 raters judge around $800$ image pairs on \textit{visual quality}, \textit{general preference}, and \textit{prompt following} using a five-point scale ($-2$ to $+2$), with model identities hidden and left--right order balanced. Against the matched DPO+25-75 baseline (\cref{tab:human_eval}), raters significantly prefer \Approach+25-75 for visual quality ($57.3\%$) and general preference ($54.1\%$). Prompt following is directionally positive ($53.4\%$) but has a high tie rate ($49.8\%$) and is not significant. To anchor these results, we add comparisons against the full-human reference DPO+100. Despite using only 25\% of the human labels, \Approach+25-75 remains close to DPO+100: general preference modestly favors DPO+100, while the confidence intervals include 50\% for visual quality and prompt following. DPO+100 itself is preferred over DPO+25-75 on all three dimensions. Human judgments thus corroborate the automatic metrics.

\begin{table}[!ht]
  \centering
  \small
  \caption{Blind pairwise human evaluation on HPSv2 prompts. Each entry is the tie-excluded win rate (\%) of the first method, with a Wilson $95\%$ confidence interval.}
  \label{tab:human_eval}
  \begin{tabular*}{\linewidth}{@{\extracolsep{\fill}}lccc@{}}
    \toprule
    \textbf{Comparison} & \textbf{Visual quality} & \textbf{General preference} & \textbf{Prompt following} \\
    \midrule
    \Approach+25-75 vs.\ DPO+25-75 & $57.3\ [53.7,\,60.8]$ & $54.1\ [50.4,\,57.7]$ & $53.4\ [49.0,\,57.8]$ \\
    DPO+100 vs.\ \Approach+25-75   & $53.7\ [49.4,\,57.8]$ & $54.0\ [50.1,\,57.9]$ & $51.0\ [44.3,\,57.5]$ \\
    DPO+100 vs.\ DPO+25-75         & $57.3\ [53.4,\,61.1]$ & $56.3\ [51.5,\,60.9]$ & $58.6\ [53.0,\,63.7]$ \\
    \bottomrule
  \end{tabular*}
\end{table}

{
  \setlength{\tabcolsep}{1pt}
  \begin{figure}[ht]
    \centering
    \begin{adjustbox}{max width=\textwidth}
      \begin{tabular}{
          m{1.2em}<{\centering}
          m{3.6cm}<{\centering}
          m{3.6cm}<{\centering}
          m{3.6cm}<{\centering}
          m{3.6cm}<{\centering}
        }
        &
        \parbox{3.6cm}{\centering\footnotesize Image of the Sandman, an ancient wizard with hour glass, casting beautiful dreams.}
        &
        \parbox{3.6cm}{\centering\footnotesize A woman with long hair next to a luminescent bird.}
        &
        \parbox{3.6cm}{\centering\footnotesize Cyberpunk girl sitting in a box in advanced anime digital art.}
        &
        \parbox{3.6cm}{\centering\footnotesize A man smoking in a 1920s speakeasy lounge; detailed painterly digital illustration; ultra-high resolution}
        \\

        \rotatebox{90}{\scriptsize \textcolor{blue}{\textbf{DeDPO + synthetic pref.}}} &
        \includegraphics[width=3.6cm]{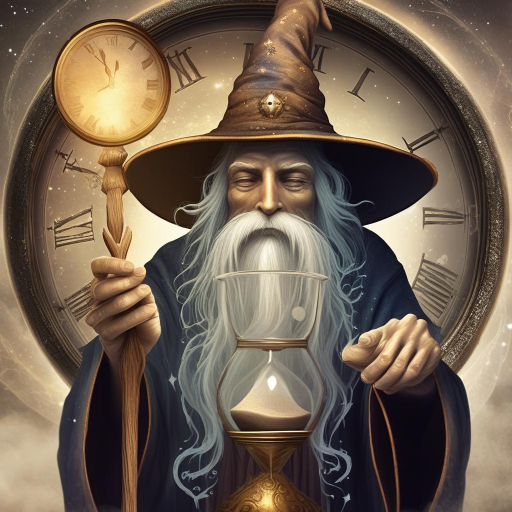} &
        \includegraphics[width=3.6cm]{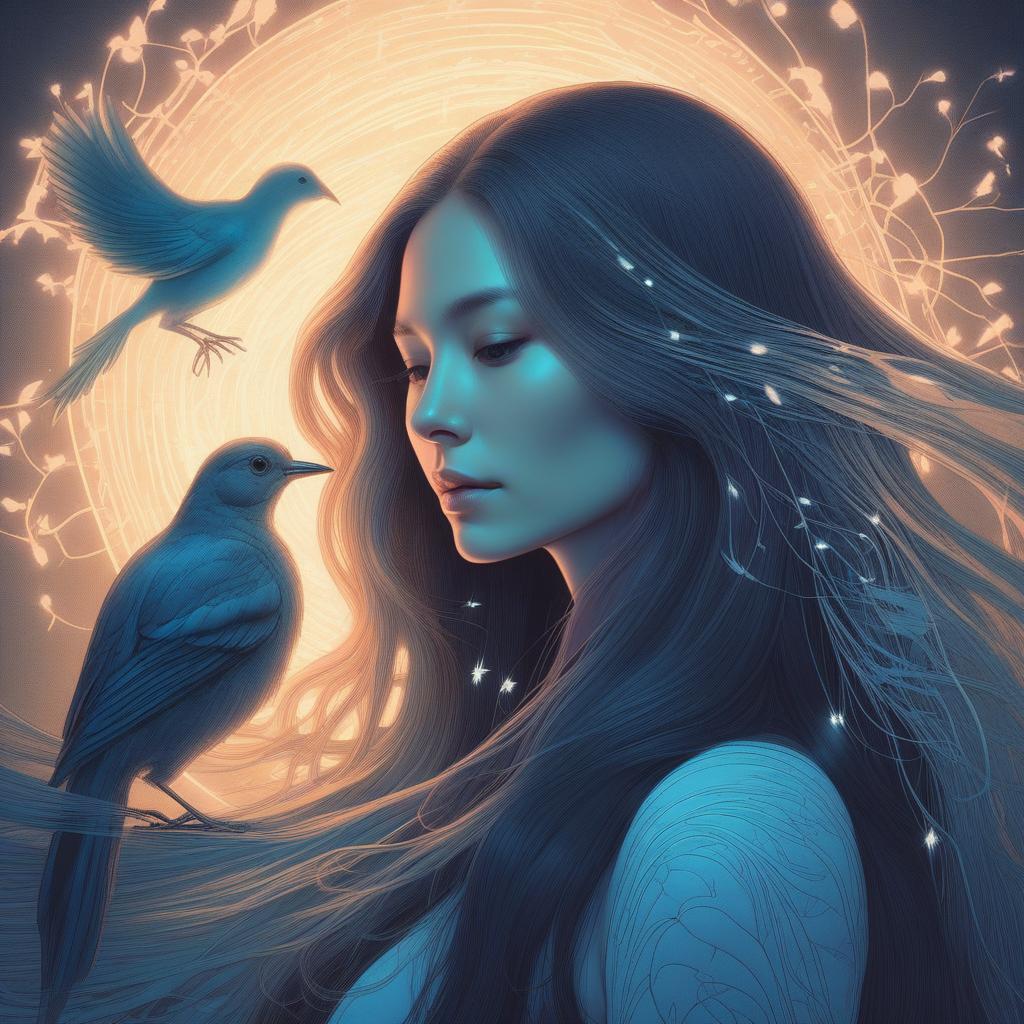} &
        \includegraphics[width=3.6cm]{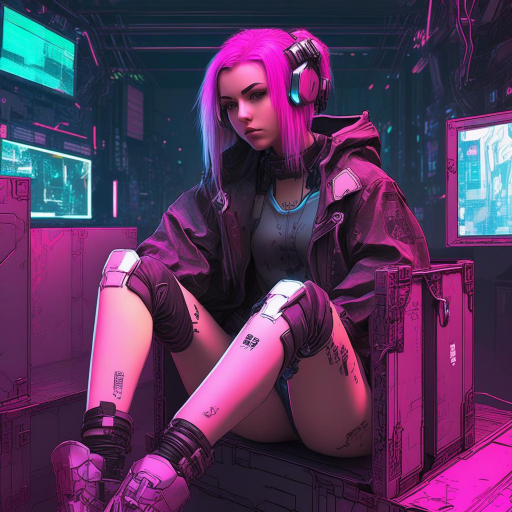} &
        \includegraphics[width=3.6cm]{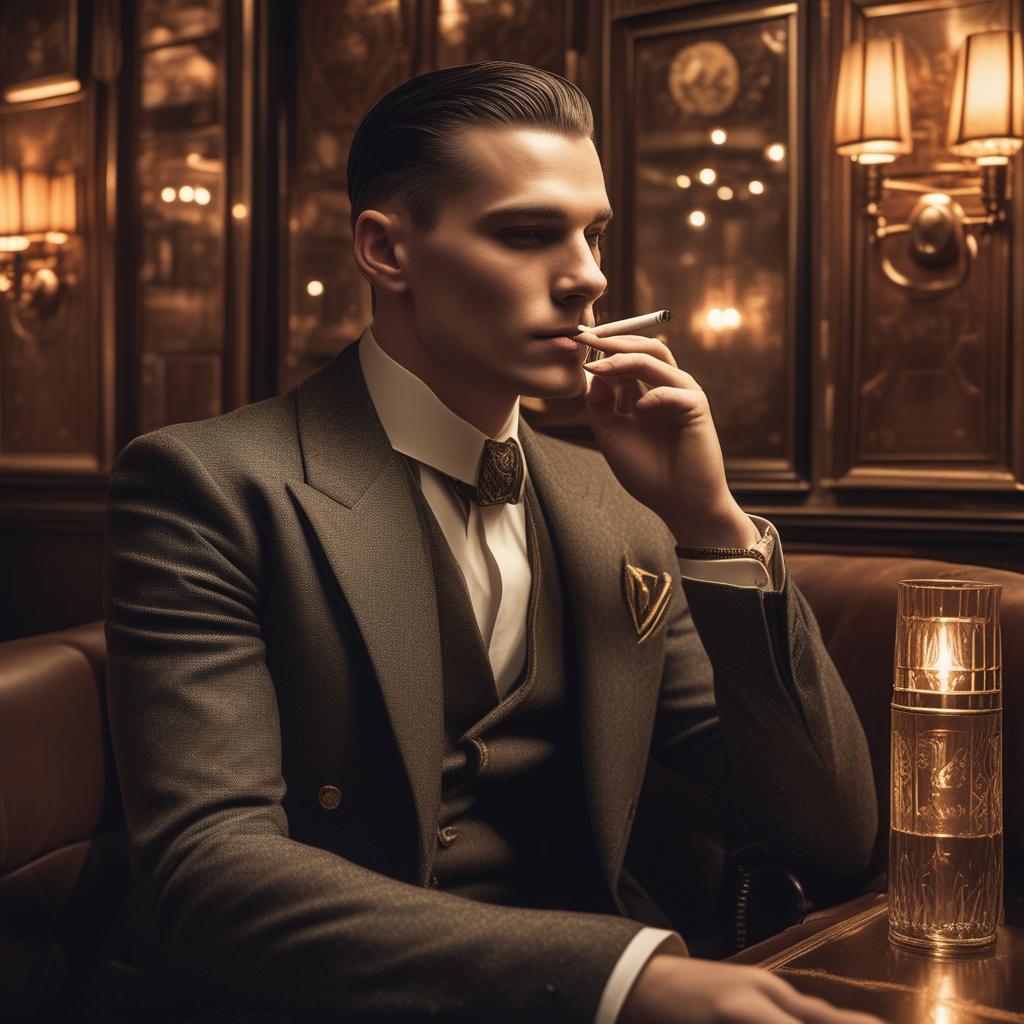}
        \\[-2pt]

        \rotatebox{90}{\scriptsize DPO + synthetic pref.} &
        \includegraphics[width=3.6cm]{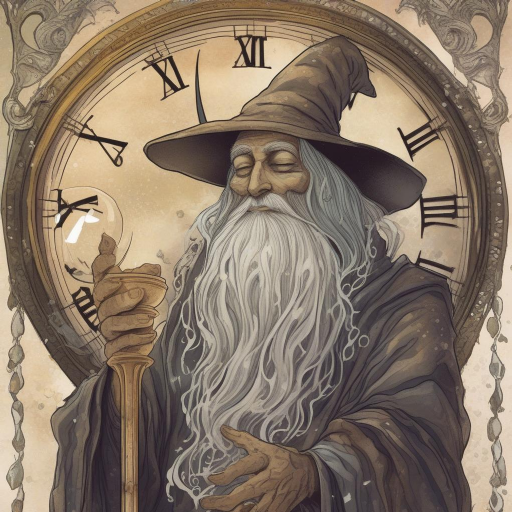} &
        \includegraphics[width=3.6cm]{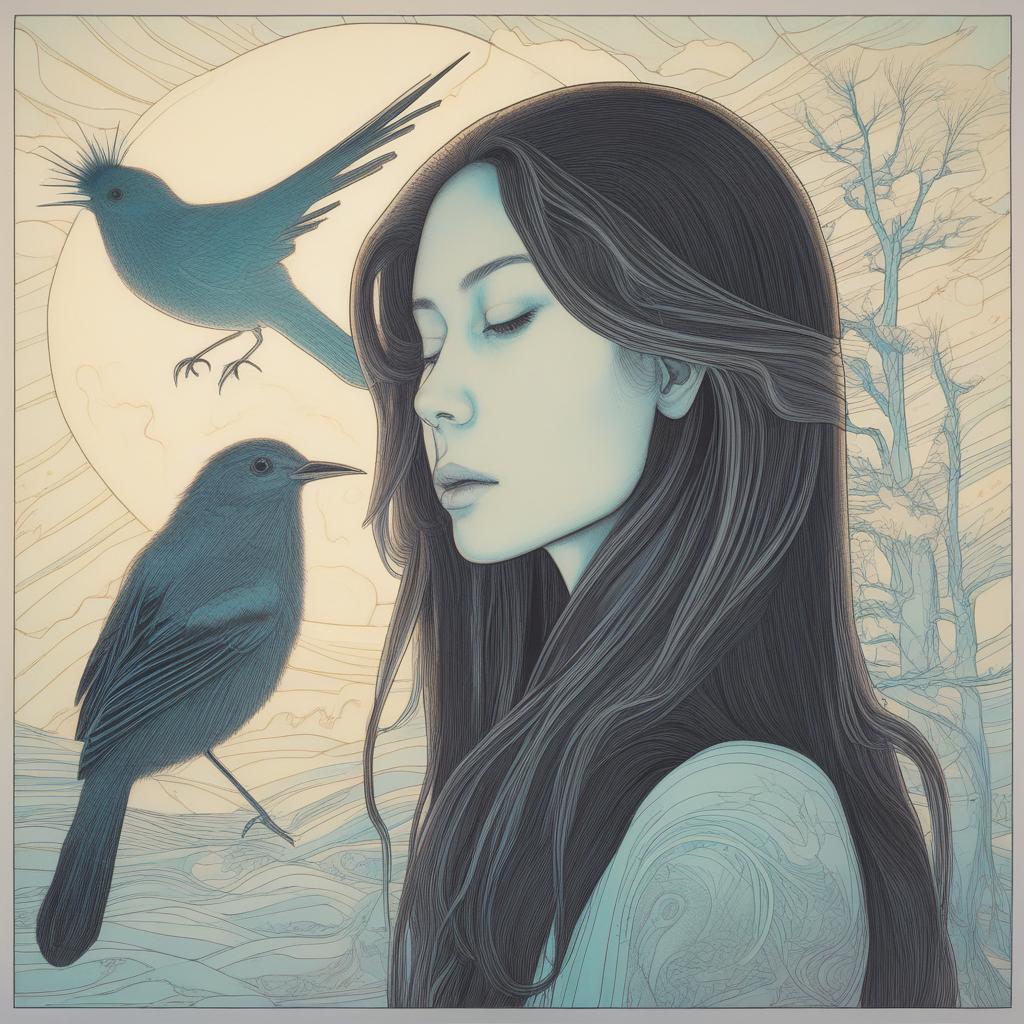} &
        \includegraphics[width=3.6cm]{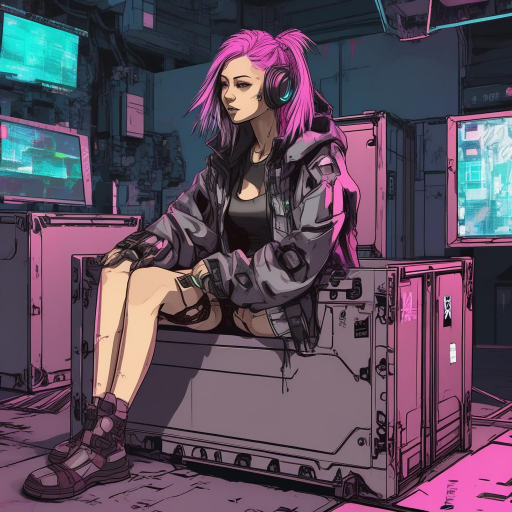} &
        \includegraphics[width=3.6cm]{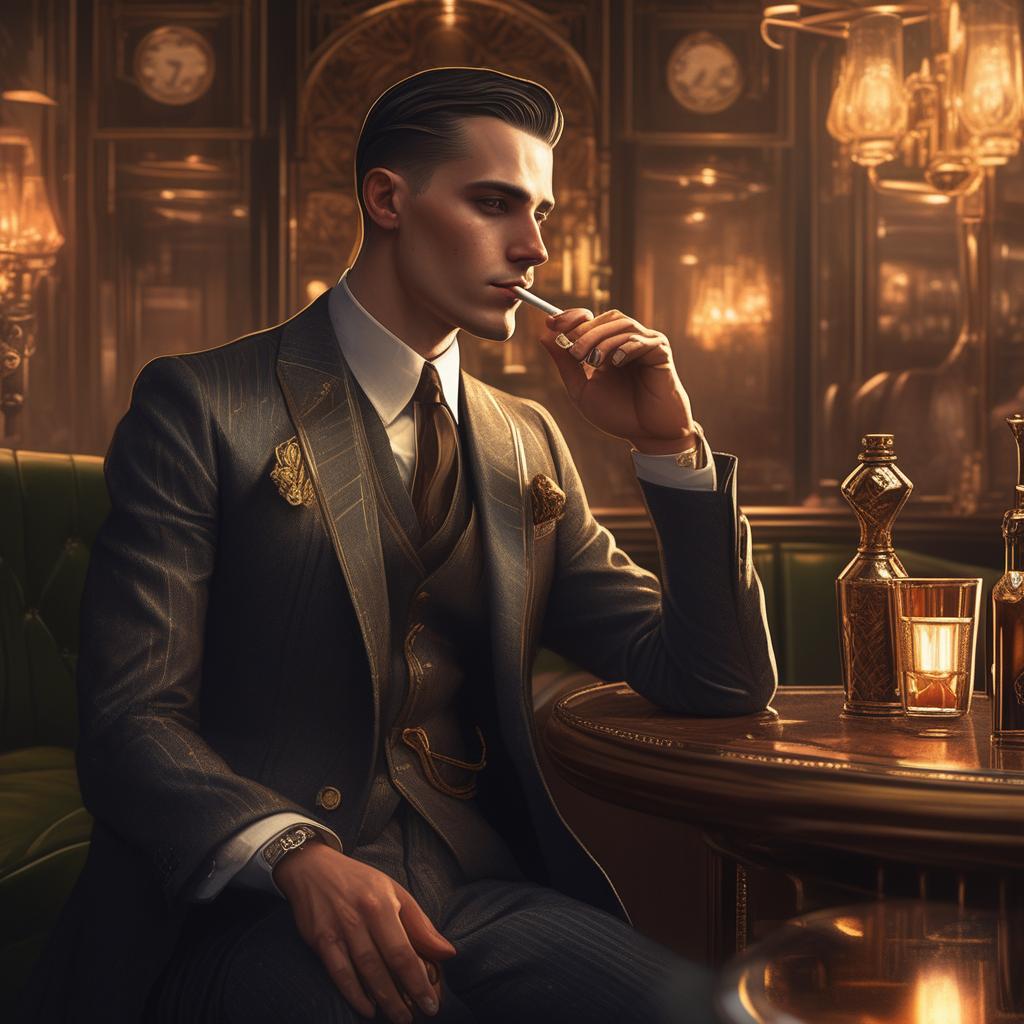}
        \\[-2pt]

        \rotatebox{90}{\scriptsize DPO + 100\% human labels} &
        \includegraphics[width=3.6cm]{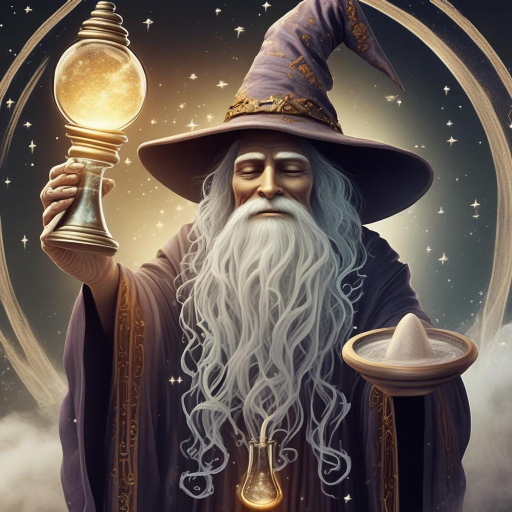} &
        \includegraphics[width=3.6cm]{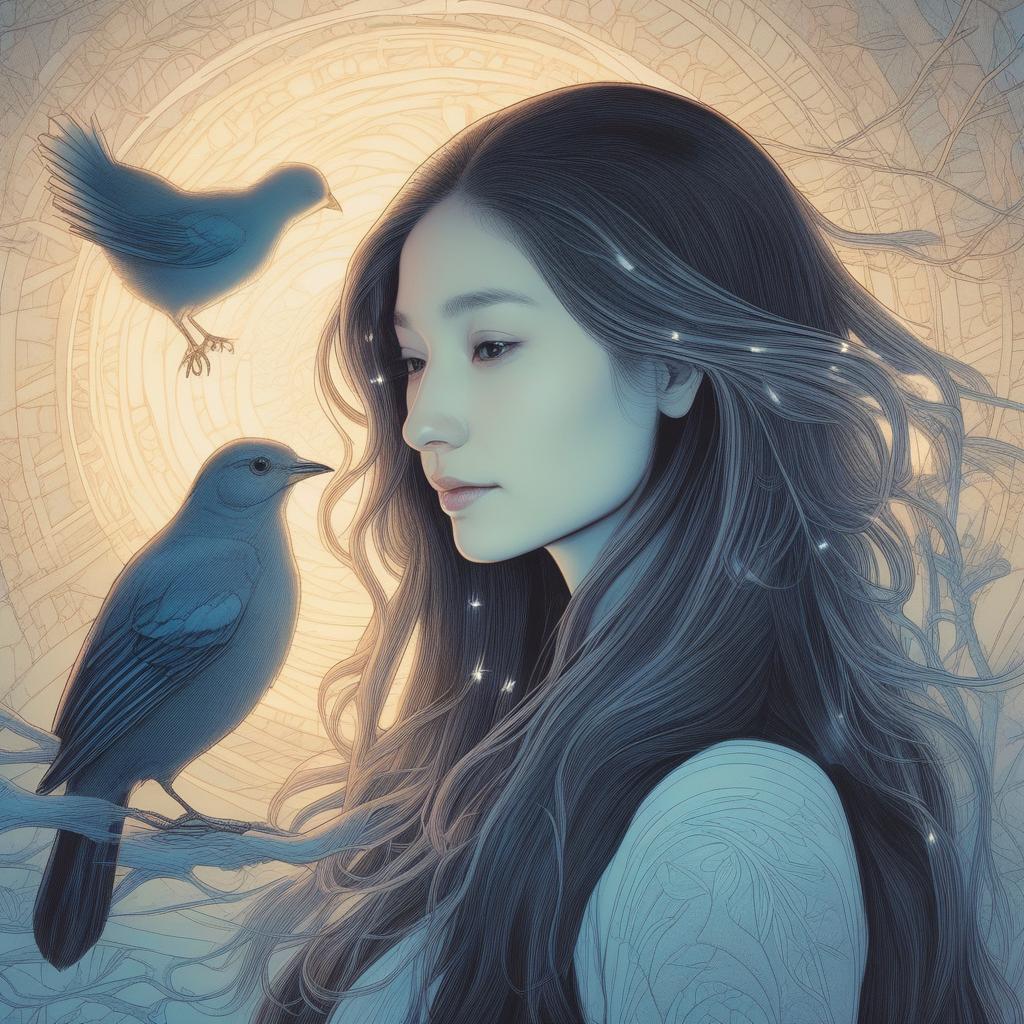} &
        \includegraphics[width=3.6cm]{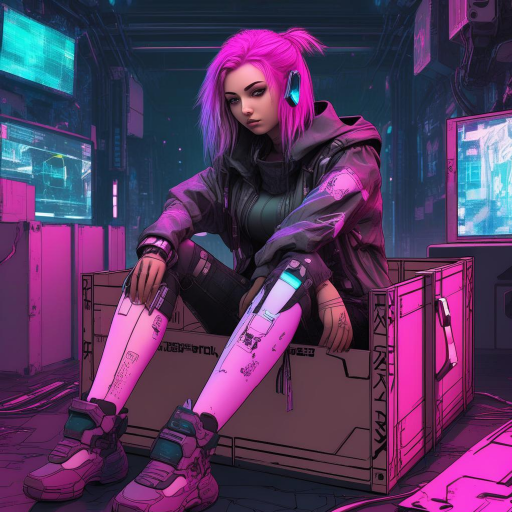} &
        \includegraphics[width=3.6cm]{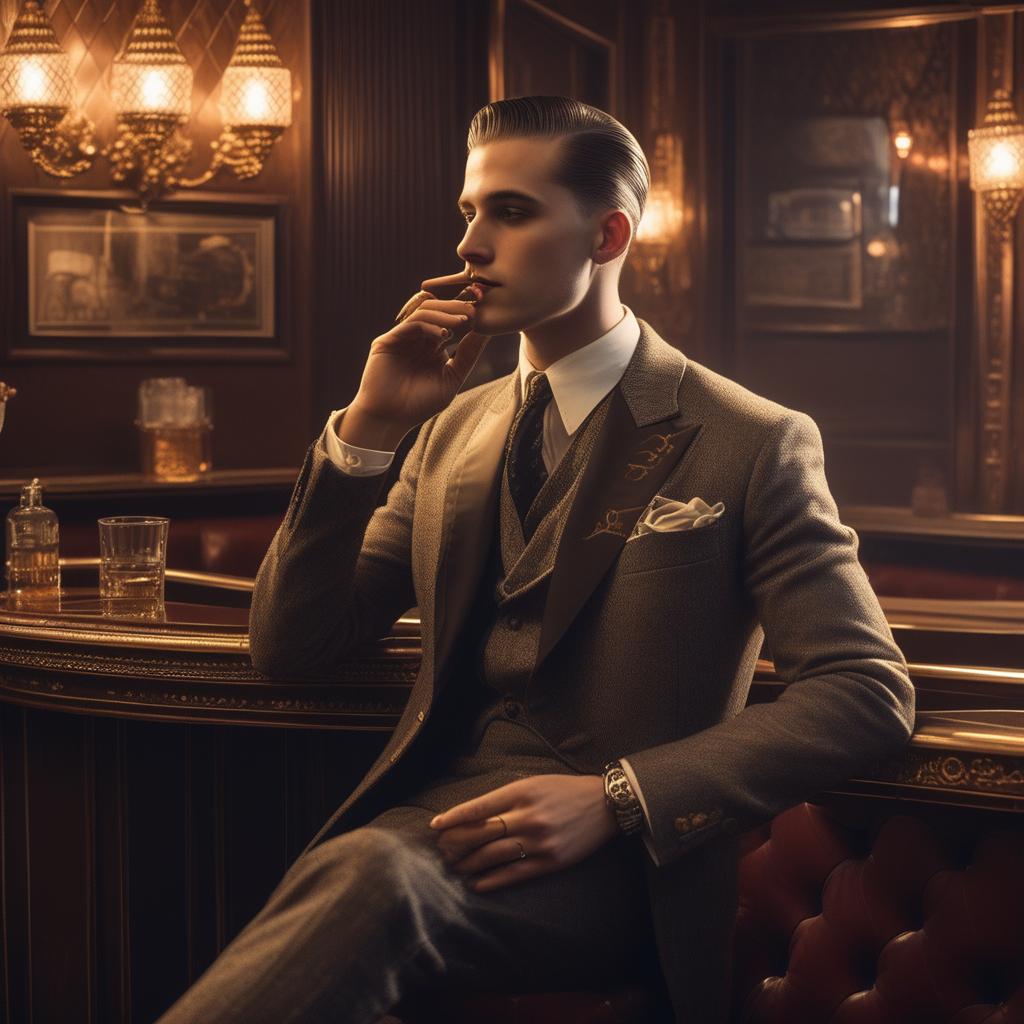}
        \\[-2pt]

      \end{tabular}
    \end{adjustbox}

    \caption{\textbf{Qualitative comparison} of DeDPO+25-75, DPO+100, and naively mixed DPO+25-75. DeDPO produces images closer to the fully supervised DPO+100 reference, while avoiding the stylized bias often introduced by naive mixing. Across examples, DeDPO is generally more photorealistic and better aligned with human preferences, consistent with \cref{fig:bias_analysis}. This suggests that DeDPO narrows the gap to fully human-supervised alignment using only 25\% of the human labels.}
    \label{fig:qual_main}
  \end{figure}
}

\myheading{Qualitative results.} Building on \cref{fig:bias_analysis}, which highlights disagreement between Qwen2.5-VL and human preferences, our qualitative evaluation in \cref{fig:qual_main} suggests that DeDPO reduces the tendency toward illustrative and stylized generations induced by the synthetic labeler. When trained on a mixed-label dataset comprising 25\% human annotations and 75\% synthetic labels from Qwen2.5-VL, DeDPO often produces more photorealistic images than DPO+25-75. These qualitative observations are consistent with the automated-metric trends, though they do not replace a human preference study.

\subsection{Analysis}
\myheading{Labeler sensitivity.}
\label{heading:labeler-sensitivity}
We study three kinds of sensitivity in the 25--75 setting: \textbf{(i) labeler sensitivity} across model-based labelers of varying quality (CLIP, SmolVLM, Self-Training, and Qwen2.5-VL), \textbf{(ii) sensitivity to systematic synthetic bias} favoring animated over photorealistic images, and \textbf{(iii) sensitivity to random label flipping}. Label accuracy denotes agreement between synthetic and human labels on the unlabeled set; for Self-Training, the reported 95\% accuracy is measured at the final DPO step. As shown in \cref{fig:labeler_sensitivity}, DeDPO is numerically above DPO across the plotted corruption types and accuracy levels, with the largest visible gaps under model-based labelers and systematic bias. Notably, DeDPO with the weakest labeler (CLIP, 50\% accuracy) remains close to DeDPO with the strongest one (Qwen2.5-VL, 83\%), suggesting reduced sensitivity to labeler quality. Self-Training~\cite{lee2025semi, liulearning} gives the strongest non-debiased baseline (21.79) but benefits least from debiasing (21.84), consistent with the double-dipping effect in \cref{thm:dr}. Under random flipping, both methods degrade smoothly, but DeDPO remains numerically above DPO, with the largest gap at 50\% accuracy (21.65 vs.\ 21.56).

\myheading{Performance under extreme sparsity.} In the extremely sparse regime (100--1250 human preference pairs out of 5K total), DeDPO is numerically above DPO at every tested label budget (\cref{fig:sparse_regime}), with the remaining pairs pseudo-labeled by Qwen2.5-VL. This directly reflects our motivation: curating reliable human preferences is expensive, so the key question is how much alignment benefit can be extracted from a very small labeled set. As more high-quality labels are added, the gap between DeDPO and DPO increases in Fig.~\ref{fig:sparse_regime}, a pattern consistent with DeDPO extracting more value from additional human annotations while continuing to leverage synthetic supervision.

\myheading{Effect of scaling synthetic preferences.}
We fix 1250 human-labeled pairs and progressively enlarge the unlabeled pool with additional PickScore-sorted Pick-a-Pic data. As shown in \cref{fig:scaling_synth}, DeDPO attains the highest plotted PickScore across the tested unlabeled-data budgets: it improves slightly up to 38.8K unlabeled samples and degrades less at 98.8K than DPO+25-75, which peaks earlier at 8.8K and then drops more noticeably. Even fully supervised DPO+100 deteriorates as more low-quality pairs are added, consistent with~\cite{yang2024automated}. These trends are consistent with DeDPO making better use of larger pseudo-labeled pools while keeping training anchored to the human-labeled subset.

\section{Conclusion}
In this work, we address the scalability bottleneck of preference optimization by proposing \Approach. Motivated by the prohibitive cost of human supervision, \Approach~integrates a debiased estimation method from causal inference to robustly leverage large volumes of inexpensive yet misaligned synthetic preference labels. Empirical results across multiple diffusion models and synthetic annotators show that \Approach~improves over naive synthetic training, with the largest gains when synthetic labels are more misaligned with human preferences, and achieves performance that remains competitive with fully human-supervised DPO while using substantially fewer human labels. Overall, these results suggest that properly debiased synthetic feedback can serve as a practical and cost-effective complement to scarce human supervision for scalable human--AI alignment.

\begin{ack}
The authors received no third-party funding or support for this work. The authors declare no competing financial interests.
\end{ack}

\bibliographystyle{plainnat}
\bibliography{main}

\clearpage
\appendix

\begin{center}
  {\LARGE \textbf{Supplementary Material}}\\[0.5em]
\end{center}

\vspace{1em}

\section{Theoretical Proofs}
\paragraph{Proof of~\cref{prop:unbiasedness}.}
\begin{proof}
  Simply reorder the terms in \cref{eq:dr} as follows:
  \begin{equation}
    \begin{aligned}
      &L_\text{\Approach}(\theta) =\
      \E_{n_l} \gL (\gG_\theta(\vy_l), \rvz_l) \\
      &+ \left[ \E_{n_l + n_u} \gL (\gG_\theta(\vy), \hat \gG(\vy)) - \E_{n_l} \gL (\gG_\theta(\vy_l), \hat \gG(\vy_l)) \right]
    \end{aligned}
  \end{equation}
  Then, the expectation of the first term is the population-level DPO loss \cref{eq:short-dpo},
  and the expectation of the second term is 0 because both are unbiased estimators of $\E[\gL (\gG_\theta(\vy), \hat \gG(\vy))]$.
\end{proof}
This shows that we can safely use synthetic preference labels while preserving an unbiased loss, leading to bias-free learning.
Moreover, when the synthetic labels are perfect, i.e., $\hat \gG(\vy) = \rvz$, the debiased loss becomes $\E_{n_l+n_u} \gL (\gG_\theta(\vy), \hat \gG(\vy))$, thereby using all of the data. When the synthetic labels are completely wrong, the loss asymptotically reduces to $\E_{n_l} \gL (\gG_\theta(\vy_l), \rvz_l)$, i.e., only the labeled data is used.

\paragraph{Proof of~\cref{prop:geometric}.}

\begin{proof}
  First, we rewrite the second term of \cref{eq:dr}:
  \begin{equation}
    \begin{aligned}
      &L_\text{\Approach}(\theta) = \E_{n_l+n_u} \gL (\gG_\theta(\vy), \hat \gG(\vy)) \\
      &+ \E_{n_l} (\gL (\gG_\theta(\vy_l), \vz_l) - \gL (\gG_\theta(\vy_l), \hat \gG(\vy_l))) \\
      &=\E_{n_l+n_u} \gL (\gG_\theta(\vy), \hat \gG(\vy)) \\
      &+\E_{n_l+n_u} \frac{n_l+n_u}{n_l}\vr \left(  \gL (\gG_\theta(\vy), \vz) -  \gL (\gG_\theta(\vy), \hat \gG(\vy))\right) \\
    \end{aligned}
  \end{equation}
  where $\vr$ is a binary variable equal to 1 if $\vy$ is labeled and 0 otherwise. Second, we note that the binary cross-entropy loss is linear in its second argument, i.e., the target label. Therefore, since all terms now share the same prediction $\gG_\theta(\vy)$, we can add, subtract, and rescale the targets with weight $\frac{n_l+n_u}{n_l}\vr$ to obtain:
  \begin{equation}
    \begin{aligned}
      &L_\text{\Approach}(\theta) = \\
      & \E_{n_l+n_u} \gL \left(\gG_\theta(\vy), \hat \gG(\vy) + \frac{n_l+n_u}{n_l}\vr (  \vz -  \hat \gG(\vy))\right)\\
      &= \E_{n_l+n_u} \gL \left(\gG_\theta(\vy), \hat\vw\right)
    \end{aligned}
  \end{equation}
  for $\hat\vw$ defined in the original claim.
\end{proof}

\subsection{Finite-sample convergence of DeDPO}
\label{sec:population-analysis}

\paragraph{Notation.}
Let $f_\theta(\vy)$ be the implicit preference logit of the diffusion model
with parameters $\theta$.
Recall that $\vy=(\vc,\vx^0,\vx^1)$ is the input and $\vz$ is the preference label.
With $\gG_\theta$ denoting the implicit preference classifier, we write
\[
  f_\theta(\vy)=\operatorname{logit}\gG_\theta(\vy)
  =\log\frac{\gG_\theta(\vy)}{1-\gG_\theta(\vy)}.
\]
In terms of the denoising model, the same logit is
\[
  \begin{aligned}
    f_\theta(\vy)=-\beta\Bigl[
      &\|\veps^1-\veps_\theta(\vx_t^1,t,\vc)\|_2^2
       -\|\veps^1-\veps_{\mathrm{ref}}(\vx_t^1,t,\vc)\|_2^2\\
      &-\|\veps^0-\veps_\theta(\vx_t^0,t,\vc)\|_2^2
       +\|\veps^0-\veps_{\mathrm{ref}}(\vx_t^0,t,\vc)\|_2^2
    \Bigr].
  \end{aligned}
\]
The timestep and noise variables are implicit in $\vy$ and in the
expectations.
Let $\mathcal F=\{f_\theta:\theta\in\Theta\}$ be the class of such logit
functions, and write $\hat f=f_{\hat\theta}$ for the fitted logit
function.

Write $\hat g=\hat{\gG}$ for the pretrained synthetic preference model and
define
\[
  g^*(\vy):=\gG^*(\vy)=\E[\vz\mid \vy].
\]
Write $R(f)=\E[\ell(f(\vy),\vz)]$ for the population risk
and let $f^*\in\operatorname*{arg\,min}_{f\in\mathcal F}R(f)$.

We use the indicator $\vr\in\{0,1\}$,
where $\vr=1$ means the human label is observed.
Define the true inverse labeling probability and its estimated version by
\[
  \gamma^*(\vy):=\frac{1}{\Pr(\vr=1\mid \vy)},
  \qquad
  \hat\gamma(\vy):=
    \frac{1}{\widehat{\Pr}(\vr=1\mid \vy)}.
\]
For constant labeling probability $p$, the true inverse
weight reduces to $\gamma^*(\vy)=1/p$.

Using the augmented target from \cref{prop:geometric}, write
$w(g,\gamma)=g(\vy)+\vr\gamma(\vy)(\vz-g(\vy))$ and
$L(f,g,\gamma)=\E[\ell(f(\vy),w(g,\gamma))]$ for the DeDPO
population loss. Here the inverse weight $\gamma(\vy)$ may depend on $\vy$.
Recall the logistic loss expressed in terms of a logit $a$ and target $b$,
\begin{equation}
  \label{eq:logistic-loss}
  \ell(a,b)
    =-b\log\sigma(a)-(1-b)\log(1-\sigma(a))
    =\log(1+e^a)-ba.
\end{equation}
The directional derivative of the population loss at $f$ in the direction $h$ is
\begin{equation}
  \label{eq:directional-derivative}
  D_fL(f,g,\gamma)[h]
    :=\left.\frac{d}{dt}L(f+t h,g,\gamma)\right|_{t=0}.
\end{equation}

\paragraph{Empirical loss and population optima.}
Write $\E_n[U]=n^{-1}\sum_{i=1}^n U(\vy_i,\vz_i,\vr_i)$ for the
empirical average, with $n=n_l+n_u$.
\[
  \hat w=w(\hat g,\hat\gamma),
  \qquad
  \widehat L(f)=\E_n[\ell(f(\vy),\hat w)].
\]
Let
\[
  \hat f\in\operatorname*{arg\,min}_{f\in\mathcal F}\widehat L(f),
  \qquad
  \bar f\in\operatorname*{arg\,min}_{f\in\mathcal F}
        L(f,\hat g,\hat\gamma).
\]

\begin{assumption}[First-order optimality]
  \label{ass:population-first-order}
  The population minimizers $\bar f$ and $f^*$ satisfy
  \[
    \begin{aligned}
      D_fL(\bar f,\hat g,\hat\gamma)[f-\bar f]&\ge 0,\\
      D_fL(f^*,g^*,\gamma^*)[f-f^*]&\ge 0
    \end{aligned}
    \qquad\text{for all }f\in\mathcal F.
  \]
\end{assumption}

This condition holds, for example, when $\mathcal F$ is convex in logit
space. For a nonconvex class it is an additional assumption on the
population optima.

\begin{theorem}[Finite-sample convergence with nuisance estimates]
  \label{thm:nuisance-convergence}
  Suppose $(\vy_i,\vz_i,\vr_i)_{i=1}^n$ are i.i.d., $\vz_i,\vr_i\in\{0,1\}$,
  and $\vr\perp \vz\mid \vy$. Assume that, for some fixed $\Gamma\ge1$,
  \[
    \Gamma^{-1}\le\Pr(\vr=1\mid \vy)\le1,
    \qquad
    0\le\hat g(\vy)\le1,
    \qquad
    1\le\hat\gamma(\vy)\le\Gamma
  \]
  almost surely, and that $(\hat g,\hat\gamma)$ are fixed or estimated
  independently of these $n$ observations.
  Let $\mathcal F$ be a fixed class of logit functions taking values in
  $[-M,M]$, for $M>0$, with VC-subgraph dimension at most $d$.
  Suppose the minimizers exist, the relevant empirical suprema and
  minimizers are measurable, and \cref{ass:population-first-order} holds.
  Put
  \[
    \lambda=\sigma(M)(1-\sigma(M)),
    \qquad
    K=(d+1)\log(en)+\log(2/\delta),
    \qquad 0<\delta<1.
  \]
  Then, conditionally on the fitted nuisance functions, with probability
  at least $1-\delta$,
  \begin{equation}
    \label{eq:population-target}
    \boxed{
    \|\hat f-f^*\|_2^2
      \le C_{M,\lambda,\Gamma}\frac K n
      +\frac{2}{\lambda^2}
        \|\hat g-g^*\|_4^2
        \|\hat\gamma-\gamma^*\|_4^2 .
    }
  \end{equation}
  The constant is independent of $n$, $d$, $\delta$.
\end{theorem}

We first establish the curvature property used in both parts of the
proof. The empirical-process step uses the concentration results of
orthogonal statistical learning~\cite{foster2023orthogonal}.
This logit-error guarantee complements the analysis of the loss gradient
at the population optimum in \cite{zhu2023doubly}.

\begin{lemma}[Strong convexity]
  \label{lem:population-strong-convexity}
  Fix $g,\gamma$, and assume that the augmented target
  $w=w(g,\gamma)$ satisfies $\E|w|<\infty$.
  For any two logit functions $f,f'$ taking values in $[-M,M]$,
  \begin{equation}
    \label{eq:population-strong-convexity}
      D_fL(f,g,\gamma)[f-f']
        -D_fL(f',g,\gamma)[f-f']
      \ge \lambda\|f-f'\|_2^2,
  \end{equation}
  where $\lambda=\sigma(M)(1-\sigma(M))>0$.
\end{lemma}

\begin{proof}
  With the augmented target $w$ held fixed, the first and second
  derivatives of the pointwise loss are
  \begin{equation}
    \label{eq:logistic-first-derivative}
    \ell'(a,w)=\sigma(a)-w
  \end{equation}
  and
  \begin{equation}
    \label{eq:logistic-second-derivative}
    \ell''(a,w)=\sigma'(a)=\sigma(a)(1-\sigma(a)).
  \end{equation}
  The second derivative is independent of $w$. On $[-M,M]$, it is
  bounded below by $\lambda$.

  Let $h=f-f'$. The chain rule gives
  \[
    \frac{d}{dt}\ell(f(\vy)+t h(\vy),w)
      =\ell'(f(\vy)+t h(\vy),w)\,h(\vy).
  \]
  Since $h$ is bounded and $\E|w|<\infty$, we can differentiate
  under the expectation. Evaluating at $t=0$ yields
  \[
    D_fL(f,g,\gamma)[h]
      =\E[(\sigma(f(\vy))-w)h(\vy)].
  \]
  Taking the difference of the derivatives at $f$ and $f'$, in the
  same direction $h$, cancels the augmented target,
  \[
      D_fL(f,g,\gamma)[h]-D_fL(f',g,\gamma)[h]
      =\E[
          (\sigma(f(\vy))-\sigma(f'(\vy)))h(\vy)].
  \]
  By the mean value theorem, for each $\vy$ there is a value $\xi_\vy$
  between $f(\vy)$ and $f'(\vy)$ such that
  \[
    \sigma(f(\vy))-\sigma(f'(\vy))
      =\sigma'(\xi_\vy)(f(\vy)-f'(\vy)).
  \]
  Since $\xi_\vy\in[-M,M]$, we have $\sigma'(\xi_\vy)\ge\lambda$, so
  \[
    (\sigma(f(\vy))-\sigma(f'(\vy)))h(\vy)
      =\sigma'(\xi_\vy)h(\vy)^2
      \ge\lambda h(\vy)^2.
  \]
  Taking expectations gives
  \begin{equation}
    \label{eq:population-curvature-bound}
    \begin{aligned}
      &D_fL(f,g,\gamma)[h]-D_fL(f',g,\gamma)[h]\\
      &\qquad\ge\lambda\E[h(\vy)^2]
        =\lambda\|f-f'\|_2^2,
    \end{aligned}
  \end{equation}
  which proves the lemma.
\end{proof}

\begin{proof}[Proof of \cref{thm:nuisance-convergence}]
Condition on $(\hat g,\hat\gamma)$ throughout the proof. Independence
allows us to treat these functions as fixed while the observations
remain i.i.d.

\paragraph{Population error.}
Let $h=\bar f-f^*$. Applying
\cref{lem:population-strong-convexity} with the same nuisance functions
at both logit functions gives
\begin{equation}
  \label{eq:population-minimizer-comparison}
    D_fL(\bar f,\hat g,\hat\gamma)[h]
    -D_fL(f^*,\hat g,\hat\gamma)[h]
    \ge\lambda\|\bar f-f^*\|_2^2.
\end{equation}

Taking $f=f^*$ in \cref{ass:population-first-order} and using
$h=\bar f-f^*$ gives
\[
  D_fL(\bar f,\hat g,\hat\gamma)[h]\le 0.
\]
Substituting this inequality into \cref{eq:population-minimizer-comparison}
yields
\begin{equation}
  \label{eq:population-optimality-bound}
  \begin{aligned}
    \lambda\|\bar f-f^*\|_2^2
      &\le -D_fL(f^*,\hat g,\hat\gamma)[h]\\
      &=-\E\!\left[
          \bigl(\sigma(f^*(\vy))-w(\hat g,\hat\gamma)\bigr)
          h(\vy)
        \right].
  \end{aligned}
\end{equation}

\paragraph{Cancellation of the first-order nuisance errors.}
Write $\hat w=w(\hat g,\hat\gamma)$ and
$w^*=w(g^*,\gamma^*)$. With the true nuisance functions,
$\E[w^*\mid \vy]=g^*(\vy)$, so $L(f,g^*,\gamma^*)=R(f)$.
Taking $f=\bar f$ in the second inequality of
\cref{ass:population-first-order} gives
\[
  D_fL(f^*,g^*,\gamma^*)[h]
    =\E[(\sigma(f^*(\vy))-w^*)h(\vy)]\ge 0.
\]
Adding this nonnegative term to the right-hand side of
\cref{eq:population-optimality-bound} cancels the sigmoid terms,
\begin{equation}
  \label{eq:population-target-difference}
  \begin{aligned}
    \lambda\|\bar f-f^*\|_2^2
      &\le -D_fL(f^*,\hat g,\hat\gamma)[h]
            +D_fL(f^*,g^*,\gamma^*)[h]\\
      &=\E[(\hat w-w^*)h(\vy)].
  \end{aligned}
\end{equation}

Define the nuisance errors $e_g=\hat g-g^*$ and
$e_\gamma=\hat\gamma-\gamma^*$. Expanding the target difference gives
\begin{equation}
  \label{eq:population-nuisance-expansion}
  \hat w-w^*
    =\vr e_\gamma(\vz-g^*)+e_g(1-\vr\gamma^*)-\vr e_\gamma e_g,
\end{equation}
where all functions on the right are evaluated at $\vy$.
The definitions of $g^*,\gamma^*$ and the assumption $\vr\perp \vz\mid \vy$
give
\[
  \E[\vr(\vz-g^*)\mid \vy]=0,
  \qquad
  \E[1-\vr\gamma^*\mid \vy]=0.
\]
Since $e_g$, $e_\gamma$, and $h$ depend only on $\vy$, the first two
terms of \cref{eq:population-nuisance-expansion} vanish after
multiplication by $h(\vy)$ and averaging. Hence
\[
  \E[(\hat w-w^*)h(\vy)]
    =-\E[\vr e_\gamma(\vy)e_g(\vy)h(\vy)],
\]
and \cref{eq:population-target-difference} implies
\begin{equation}
  \label{eq:population-nuisance-product}
  \lambda\|\bar f-f^*\|_2^2
    \le\left|\E[\vr e_\gamma(\vy)e_g(\vy)h(\vy)]\right|.
\end{equation}
Thus only the product of the two nuisance errors remains.

Finally, since $0\le \vr\le 1$, Hölder's inequality with exponents
$4,4,2$ gives
\begin{equation}
  \label{eq:population-holder-bound}
  \begin{aligned}
    \lambda\|\bar f-f^*\|_2^2
      &\le\E[|e_\gamma(\vy)e_g(\vy)h(\vy)|]\\
      &\le\|\hat g-g^*\|_4
            \|\hat\gamma-\gamma^*\|_4
            \|\bar f-f^*\|_2.
  \end{aligned}
\end{equation}
Dividing when $\|\bar f-f^*\|_2>0$ and squaring gives the following
bound, which is immediate when the norm is zero:
\begin{equation}
  \label{eq:population-error-complete}
  \|\bar f-f^*\|_2^2
    \le\frac{1}{\lambda^2}
       \|\hat g-g^*\|_4^2\|\hat\gamma-\gamma^*\|_4^2.
\end{equation}

\paragraph{Empirical error.}
Abbreviate $L(f)=L(f,\hat g,\hat\gamma)$.
For each $f\in\mathcal F$ define the pointwise loss difference
\[
  U_f(\vy,\vz,\vr)=\ell(f(\vy),\hat w)-\ell(\bar f(\vy),\hat w),
  \qquad
  \hat w=w(\hat g,\hat\gamma).
\]
Thus $\E[U_f]=L(f)-L(\bar f)$ and
$\E_n[U_f]=\widehat L(f)-\widehat L(\bar f)$.
Empirical optimality gives $\E_n[U_{\hat f}]\le0$, so
\begin{equation}
  \label{eq:empirical-basic-inequality}
  L(\hat f)-L(\bar f)
    \le (\E-\E_n)U_{\hat f}.
\end{equation}

Let $A=\Gamma$. The bounds on $\hat g$ and $\hat\gamma$ imply
$1-A\le\hat w\le A$. Hence
\[
  |\ell'(a,\hat w)|=|\sigma(a)-\hat w|\le A,
  \qquad a\in[-M,M].
\]
The mean value theorem and the definition of variance give
\begin{equation}
  \label{eq:contrast-variance-bound}
  \begin{aligned}
    |U_f|&\le A|f(\vy)-\bar f(\vy)|\le2MA,\\
    \operatorname{Var}(U_f)
      &\le \E[U_f^2]
       \le A^2\|f-\bar f\|_2^2.
  \end{aligned}
\end{equation}

We apply the uniform concentration bound in
\cite[Appendix~K, Lemma~14]{foster2023orthogonal}. We show how the
VC-subgraph assumption supplies its complexity condition. Put
$B=\max\{1,M\}$ and normalize the logit class by $B$. For every probability
measure $Q$, the VC-subgraph covering bound gives
\[
  \log N(\varepsilon,\mathcal F/B,L_2(Q))
    \le C(d+1)\log(C/\varepsilon),
    \qquad 0<\varepsilon<1,
\]
where $N$ is the minimum number of $L_2(Q)$ balls of radius
$\varepsilon$ covering the class. Translation by $\bar f/B$ preserves
covering numbers. Passing to
\[
  \mathcal H=\{t(f-\bar f)/B:f\in\mathcal F,\ 0\le t\le1\}
\]
adds only a term of order $\log(C/\varepsilon)$ to their logarithms.
An additional normalization by $2$ bounds $\mathcal H$ by $1$ and changes
only universal constants in the cited entropy calculations.
In particular, the entropy integral for a neighborhood of radius
$0<s<1$ is bounded by
\[
  \int_{s^2/8}^{s}
    \sqrt{\frac{C(d+1)\log(C/\varepsilon)}{n}}\,d\varepsilon
    \le C s\sqrt{\frac{(d+1)\log(C/s)}{n}}.
\]
Consequently, $\tau=C_0\sqrt{K/n}$, with $C_0$ sufficiently large,
satisfies the critical-radius and confidence requirements in the cited
lemma. This is the VC entropy calculation of
\cite[Appendix~H.1.1, Lemmas~8--10]{foster2023orthogonal}, applied to the
normalized class. When this choice is larger than a fixed constant,
the desired theorem bound follows directly from
$\|\hat f-f^*\|_2^2\le4M^2$, after enlarging its constant.

On the normalized class, the loss $(a,\hat w)\mapsto\ell(Ba,\hat w)$
is $AB$-Lipschitz. Lemma~14 therefore yields, with probability at least
$1-\delta$, simultaneously for all $f\in\mathcal F$,
\begin{equation}
  \label{eq:empirical-uniform-bound}
  \begin{aligned}
    |(\E_n-\E)U_f|
      &\le C AB\tau
        \left(\frac{\|f-\bar f\|_2}{B}+\tau\right)\\
      &\le C A\left[
          \|f-\bar f\|_2\sqrt{\frac K n}
          +B\frac K n\right].
  \end{aligned}
\end{equation}
Conditioning makes $\bar f$ a fixed comparator for this application.
The uniform event permits substituting the data-dependent $\hat f$.

For completeness, the curvature lower bound also gives a loss-gap
bound. For $h=f-\bar f$, Taylor's integral formula gives
\[
  \begin{aligned}
    L(f)-L(\bar f)
      &=D_fL(\bar f)[h]
        +\int_0^1(1-t)
           \E[\sigma'(\bar f(\vy)+t h(\vy))h(\vy)^2]\,dt\\
      &\ge\frac{\lambda}{2}\|f-\bar f\|_2^2.
  \end{aligned}
\]
The first-order term is nonnegative by
\cref{ass:population-first-order}; the line segment between bounded
logits stays in $[-M,M]$. Evaluating the loss along that segment does
not require the intermediate functions to belong to $\mathcal F$.
Combine this inequality with
\cref{eq:empirical-basic-inequality,eq:empirical-uniform-bound} and set
$s=\|\hat f-\bar f\|_2$ to obtain
\[
  \frac{\lambda}{2}s^2
    \le C A s\sqrt{\frac K n}+C AB\frac K n.
\]
Young's inequality gives
\[
  C A s\sqrt{\frac K n}
    \le\frac{\lambda}{4}s^2
       +\frac{C^2A^2}{\lambda}\frac K n.
\]
Subtracting the quadratic term and dividing by $\lambda/4$ yields
\begin{equation}
  \label{eq:empirical-error-complete}
  \|\hat f-\bar f\|_2^2
    \le 4\left(\frac{C^2A^2}{\lambda^2}
               +\frac{CAB}{\lambda}\right)\frac K n
    =C_{M,\lambda,\Gamma}\frac K n.
\end{equation}

\paragraph{Combining the two errors.}
Finally,
\[
  \begin{aligned}
    \|\hat f-f^*\|_2^2
      &\le2\|\hat f-\bar f\|_2^2+2\|\bar f-f^*\|_2^2\\
      &\le C_{M,\lambda,\Gamma}\frac K n
        +\frac{2}{\lambda^2}
           \|\hat g-g^*\|_4^2\|\hat\gamma-\gamma^*\|_4^2,
  \end{aligned}
\]
where the statistical constant has been enlarged to absorb the factor
$2$. This proves the conditional result; averaging the conditional
failure probability gives the same guarantee unconditionally.
\end{proof}

\paragraph{Proof of \cref{cor:propensity-cases}.}
\begin{proof}
  When $\hat\gamma=\gamma^*$, \cref{eq:population-target} gives
  $\|\hat f-f^*\|_2^2\le C_{M,\lambda,\Gamma}K/n$ with probability
  at least $1-\delta$.
  For estimated propensity, let $\mathcal E$ be an event determined by
  the fitted nuisances such that $\Pr(\mathcal E)\ge1-\eta$ and, on
  $\mathcal E$,
  \[
    \|\hat\gamma-\gamma^*\|_4\|\hat g-g^*\|_4\le b n^{-1/2}.
  \]
  Conditional on any such fitted nuisances, the theorem gives
  \[
    \|\hat f-f^*\|_2^2
      \le C_{M,\lambda,\Gamma}\frac K n
         +\frac{2b^2}{\lambda^2 n}
  \]
  with probability at least $1-\delta$. A union bound with
  $\mathcal E^\vc$ gives joint probability at least $1-\delta-\eta$.
\end{proof}

\paragraph{Fourth-order nuisance dependence.}
If, for a constant $C_0$,
\[
  \|\hat\gamma-\gamma^*\|_4\le C_0\|\hat g-g^*\|_4,
\]
then the product term satisfies
\[
  \|\hat\gamma-\gamma^*\|_4^2\|\hat g-g^*\|_4^2
    \le C_0^2\|\hat g-g^*\|_4^4.
\]
This condition is needed only for the fourth-order specialization;
the theorem itself imposes no convergence rate on either nuisance.

The theorem is a finite-sample guarantee for the empirical logit
minimizer. With a
fixed function-class dimension, its statistical term is of order
$1/n$ up to logarithmic factors.

\section{Experiment Setup}

\subsection{Implementation of \Approach}

To make the method concrete, Fig.~\ref{fig:debiased_dpo} shows a minimal implementation of \Approach. In practice, \Approach~uses exactly the same logits and binary cross-entropy components as standard Diffusion-DPO, and differs only by an additional correction term applied on the labeled subset. Consequently, moving from the original DPO objective to our debiased objective requires only a small code change and introduces essentially no additional algorithmic complexity. We include this implementation sketch to complement the theory and to emphasize that the practical gains of \Approach~do not depend on a complicated training pipeline.

\begin{figure}[htbp]
  \centering
\begin{lstlisting}[
    language=Python,
    basicstyle=\ttfamily\small,
    commentstyle=\color{blue},
    keywordstyle=\bfseries,
    numbers=left,
    numberstyle=\tiny\color{gray},
    stepnumber=1,
    numbersep=8pt,
    frame=single,
    keepspaces=true,
    columns=fullflexible
]
# model_diff   : policy MSE difference (image_1 - image_0), shape (B,)
# ref_diff     : reference MSE difference (image_1 - image_0), shape (B,)
# y_true_L     : ground truth preferences for labeled subset, shape (B_L,)
# y_pseudo     : external model pseudo-labels for entire batch, shape (B,)
# beta         : KL regularization coefficient
# B_L, B_U     : labeled and unlabeled mini-batch sizes
# wt           : inverse labeled sampling fraction, (B_L + B_U) / B_L

# Calculate Diffusion DPO logits
logits = -0.5 * beta * (model_diff - ref_diff)

# Split batches into Labeled (L) and Unlabeled (U) subsets
logits_L, logits_U = logits[:B_L], logits[B_L:]
y_pseudo_L, y_pseudo_U = y_pseudo[:B_L], y_pseudo[B_L:]

# Component BCE losses (summed over each subset)
loss_L = F.binary_cross_entropy_with_logits(
    logits_L, y_true_L, reduction="sum")
loss_pseudo_L = F.binary_cross_entropy_with_logits(
    logits_L, y_pseudo_L, reduction="sum")
loss_pseudo_U = F.binary_cross_entropy_with_logits(
    logits_U, y_pseudo_U, reduction="sum")

# Standard Semi-Supervised Diffusion DPO
SemiSup_DPO_loss = loss_L / B_L + loss_pseudo_U / B_U

# Debiased Diffusion DPO
Debiased_DPO_loss = (
    loss_pseudo_L + loss_pseudo_U
    + wt * (loss_L - loss_pseudo_L)
) / (B_L + B_U)
\end{lstlisting}
  \caption{Pseudocode for Debiased Diffusion DPO, shown alongside the standard semi-supervised objective for comparison. The code allows arbitrary positive batch sizes $B_L,B_U$, with $w_t=(B_L+B_U)/B_L$. Our experiments use $B_L=B_U$ and $w_t=2$; see \cref{app:minibatch}.}
  \label{fig:debiased_dpo}
\end{figure}

\myheading{Balanced mini-batches and loss weighting.}
\label{app:minibatch}
Let $B_L$ and $B_U$ denote the numbers of labeled and unlabeled pairs in a mini-batch, with total size $B=B_L+B_U$, and let $q_L=B_L/B$ be its labeled sampling fraction. This fraction is distinct from the dataset fraction $p=n_l/(n_l+n_u)$. The correction weight is $w_t=1/q_L$. Writing $\ell_L$, $\ell_{\hat L}$, and $\ell_{\hat U}$ for the mean human-label loss on labeled pairs and the mean pseudo-label losses on labeled and unlabeled pairs, respectively, the mini-batch loss is
\[
  \widehat L_{\text{\Approach}}
  =q_L\ell_{\hat L}+(1-q_L)\ell_{\hat U}
   +q_Lw_t(\ell_L-\ell_{\hat L}).
\]
The code computes the component losses with \texttt{reduction="sum"} and divides the combined DeDPO loss by $B$, yielding this same expression.
In all experiments, we use equal numbers of labeled and unlabeled pairs in each mini-batch, regardless of the pool sizes $(n_l,n_u)$. We obtain this balance by undersampling the larger pool or oversampling the smaller pool, giving $B_L=B_U$, $q_L=1/2$, and $w_t=2$. The loss then reduces to $\tfrac12[\ell_{\hat L}+\ell_{\hat U}+2(\ell_L-\ell_{\hat L})]$. Balanced batches ensure that the human correction is represented in every update. Each reuse of a pair independently resamples the diffusion timestep and Gaussian noise. Under the random split used in our experiments, labeled and unlabeled pairs have the same population distribution, so this rebalancing preserves the population target of \cref{eq:dr}.

\subsection{Datasets}

\myheading{Pick-a-Pic v2.} Pick-a-Pic \cite{Kirstain2023PickaPicAO} is collected from interactions in a public text-to-image web application: a user writes a prompt, is shown two generated images, and selects a preferred image or a tie; the rejected image is then replaced and the process continues. The original paper reports more than 500k examples and 35k distinct prompts in the release used there, and also notes that the newer v2 release contains more than one million examples. The images are generated from multiple backbones and classifier-free guidance settings, and the data collection process is centered on real, intrinsically motivated users rather than paid crowd workers. This provenance is both the main strength and the main caveat of PaP. On the positive side, it offers strong ecological validity, broad prompt diversity, and direct human pairwise supervision. On the other hand, it naturally reflects the preferences of a self-selected online creator community and the biases induced by the app interaction loop itself, so its prompt distribution, stylistic preferences, and ambiguity patterns need not match those of more balanced benchmark-style datasets.

\myheading{HPDv2.} HPDv2 \cite{wu2023human} was introduced as a larger and cleaner human-preference dataset designed to mitigate several known biases in earlier resources. The paper reports 798,090 preference choices over 433,760 image pairs, with prompts and images deliberately drawn from diverse sources to improve coverage and reduce dataset bias. The official repository provides additional context: prompts are sourced from DiffusionDB~\cite{wang2023diffusiondb} and MSCOCO captions~\cite{lin2014microsoft}; DiffusionDB prompts are first cleaned with ChatGPT to remove biased function words; and human annotators then rank images generated by different text-to-image models for the same prompt. Relative to Pick-a-Pic, this makes HPDv2 a stronger clean-label source, but not a bias-free one. Its prompt cleaning and curated benchmark design improve stability and breadth, but they also make the resulting distribution more controlled and less representative of unconstrained real-user prompting. In that sense, HPDv2 should be viewed as less noisy and more curated, rather than as perfectly unbiased.

\myheading{On PickScore-based data curation.}
A potential concern is that using PickScore for training-data selection may introduce circularity when PickScore is also reported as one of the evaluation metrics. We clarify that PickScore is used only as a coarse curation signal before training, not as supervision for the preference objective and not as part of the test-time evaluation pipeline. Specifically, we use PickScore to rank candidate preference pairs and select a higher-quality 5K pool from each source dataset. This filtering step is intended to remove highly ambiguous or low-quality pairs so that the small human-labeled subset provides a cleaner anchor for debiasing. After this pool is formed, the actual training labels are either the original human preference labels or synthetic pairwise judgments produced by Qwen2.5-VL; PickScore is never used to assign pseudo-labels or to decide the preferred image in a training pair. This distinction avoids direct circularity. The evaluation prompts are drawn from the held-out PartiPrompt benchmark, and each method is evaluated on newly generated images, not on the curated training pairs. Thus, PickScore-based sorting does not expose the model to the evaluation prompts or evaluation images. Moreover, the same curated pools are used for all competing methods, including \textbf{DPO+25}, \textbf{DPO+25-75}, and \Approach. Therefore, any effect of PickScore-based selection is shared across baselines and cannot by itself explain the relative improvement of \Approach~over the direct semi-supervised baseline. The comparison between \Approach~and \textbf{DPO+25-75} is especially controlled: both methods use the same human/synthetic split and differ only in whether the debiasing correction is applied.

\subsection{Evaluation metrics.} We evaluate text-to-image preference alignment using \textbf{PickScore}, \textbf{HPSv3}, and \textbf{ImageReward}, instead of relying on older CLIPScore- or aesthetic-based metrics. Although CLIPScore and aesthetic predictors were widely used in earlier text-to-image work, they are only indirect proxies for human preference. CLIPScore primarily measures image--text representation similarity and can miss failures in realism, visual quality, or fine-grained composition, while aesthetic predictors are image-only models and therefore cannot determine whether an otherwise attractive image actually follows the prompt. More recent preference-modeling work suggests that evaluators trained on human judgments better align with human preferences than these earlier proxies. Accordingly, we use PickScore~\cite{Kirstain2023PickaPicAO}, which is trained from real user pairwise preferences; HPSv3~\cite{ma2025hpsv3}, a recent broad-coverage human preference metric intended to better capture differences across modern image generators and quality regimes; and ImageReward~\cite{xu2023imagereward}, a general-purpose reward model trained from expert human comparisons. For PickScore and HPSv3, we report the mean score over all PartiPrompt test prompts. For ImageReward, we report \textbf{ImageReward@Top-1(\%)}, defined as the percentage of prompts for which a method attains the highest ImageReward score among all compared methods. This metric is best interpreted as a prompt-level win rate: rather than averaging reward magnitudes across prompts, it measures how often a method is ranked first on a per-prompt basis. We include it as a complementary relative metric because it is less sensitive to score calibration across prompts while still capturing how consistently a model wins against the baselines on a common evaluation set. Therefore, improvements in ImageReward@Top-1 indicate that gains are distributed broadly across prompts, rather than being driven by a small number of especially favorable examples.

\section{Additional Results}
\label{app:additional}

\subsection{Training without PickScore-based curation}
\label{app:uncurated}
PickScore is used to rank candidate training pairs for PaP-5K and HPD-5K and is also one of our evaluation metrics. To rule out an interaction between this curation rule and the method, we remove PickScore ranking entirely and sample a uniformly random 5K subset of Pick-a-Pic v2. \cref{tab:uncurated} shows that \Approach~retains its advantage over the matched DPO+25-75 baseline on both a U-Net (SD1.5) and a DiT (PixArt-$\Sigma$ at $512$ resolution). The gain is largest on ImageReward@Top-1, which is independent of PickScore ($+5.52$ and $+6.44$ points), so the improvement cannot be attributed solely to PickScore-based curation.

\begin{table}[h]
  \centering
  \small
  \caption{Results on a uniformly sampled, unfiltered 5K subset of Pick-a-Pic v2 (no PickScore ranking).}
  \label{tab:uncurated}
  \begin{tabular}{llcc}
    \toprule
    \textbf{Model} & \textbf{Method} & \textbf{PickScore} ($\uparrow$) & \textbf{ImageReward@Top-1 (\%)} ($\uparrow$) \\
    \midrule
    \multirow{4}{*}{SD1.5} & DPO+25 & 21.67 & 19.11 \\
     & DPO+25-75 & 21.69 & 23.13 \\
     & \Approach+25-75 & \textbf{21.75} & \textbf{28.65} \\
     & \textcolor{gray!80}{DPO+100} & \textcolor{gray!80}{21.75} & \textcolor{gray!80}{29.11} \\
    \midrule
    \multirow{4}{*}{PixArt-$\Sigma$-512} & DPO+25 & 22.34 & 18.99 \\
     & DPO+25-75 & 22.57 & 22.23 \\
     & \Approach+25-75 & \textbf{22.58} & \textbf{28.67} \\
     & \textcolor{gray!80}{DPO+100} & \textcolor{gray!80}{22.67} & \textcolor{gray!80}{30.11} \\
    \bottomrule
  \end{tabular}
\end{table}

\subsection{DiT backbone}
\label{app:dit}
All main experiments use U-Net backbones (SD1.5 and SDXL). To test whether the effect is specific to U-Nets, we train PixArt-$\Sigma$, a DiT-based text-to-image model, on PaP-5K at $512$ and $1024$ resolution. As shown in \cref{tab:dit}, \Approach~again improves over DPO+25-75 and approaches the DPO+100 full-human reference at both resolutions, e.g., at $512$ resolution, PickScore $22.72$ vs.\ $22.66$ and ImageReward@Top-1 $29.78$ vs.\ $20.75$.

\begin{table}[h]
  \centering
  \small
  \caption{Results with a DiT backbone (PixArt-$\Sigma$) trained on PaP-5K.}
  \label{tab:dit}
  \begin{tabular}{lcccc}
    \toprule
     & \multicolumn{2}{c}{\textbf{PixArt-$\Sigma$-512}} & \multicolumn{2}{c}{\textbf{PixArt-$\Sigma$-1024}} \\
    \cmidrule(lr){2-3}\cmidrule(lr){4-5}
    \textbf{Method} & \textbf{PickScore} ($\uparrow$) & \textbf{IR@Top-1 (\%)} ($\uparrow$) & \textbf{PickScore} ($\uparrow$) & \textbf{IR@Top-1 (\%)} ($\uparrow$) \\
    \midrule
    DPO+25 & 22.60 & 19.30 & 22.81 & 20.71 \\
    DPO+25-75 & 22.66 & 20.75 & 22.84 & 24.02 \\
    \Approach+25-75 & \textbf{22.72} & \textbf{29.78} & \textbf{22.87} & \textbf{26.71} \\
    \textcolor{gray!80}{DPO+100} & \textcolor{gray!80}{22.75} & \textcolor{gray!80}{30.17} & \textcolor{gray!80}{22.89} & \textcolor{gray!80}{28.56} \\
    \bottomrule
  \end{tabular}
\end{table}

\subsection{Fraction of the full-human gap recovered}
\label{app:gap}
Raw differences in PickScore or ImageReward are metric-specific and do not indicate how much of the distance to full human supervision has been closed. For example, on PaP-5K with SD1.5, the entire PickScore improvement from SFT to DPO+100 is only $0.24$ points ($21.64\rightarrow21.88$). We therefore report the gap recovered by \Approach,
$\frac{\text{\Approach} - \text{DPO+25-75}}{\text{DPO+100} - \text{DPO+25-75}}$, capped at $100\%$.
\cref{tab:gap} shows that \Approach~recovers $96$--$100\%$ of the PickScore gap on SD1.5 and SDXL and a majority of the gap in most other cells. The lowest value (ImageReward@Top-1 on PaP-5K with SD1.5, $39\%$) occurs where the full-human headroom is largest; we report all settings rather than selecting favorable ones.

\begin{table}[h]
  \centering
  \small
  \caption{Fraction of the DPO+25-75$\rightarrow$DPO+100 gap recovered by \Approach+25-75 (capped at $100\%$). SD1.5 and SDXL rows use the 15-seed means of \cref{tab:main}; PixArt-$\Sigma$ rows use the single-run results of \cref{tab:dit}. $^\dagger$DPO+100 does not exceed DPO+25-75 in this cell, so the gap is not positive; \Approach~scores above both.}
  \label{tab:gap}
  \begin{tabular}{lcc}
    \toprule
    \textbf{Setting} & \textbf{PickScore gap} & \textbf{ImageReward@Top-1 gap} \\
    \midrule
    PaP-5K / SD1.5 & 100\% & 39\% \\
    PaP-5K / SDXL & 96\% & 91\% \\
    HPD-5K / SD1.5 & 100\%$^\dagger$ & 90\% \\
    HPD-5K / SDXL & 100\% & 82\% \\
    PaP-5K / PixArt-$\Sigma$-512 & 67\% & 96\% \\
    PaP-5K / PixArt-$\Sigma$-1024 & 60\% & 59\% \\
    \bottomrule
  \end{tabular}
\end{table}

\subsection{Balanced 50/50 human/synthetic split}
\label{app:5050}
Our main setting targets the label-scarce regime, and \cref{fig:sparse_regime,fig:scaling_synth} already sweep human/synthetic ratios from roughly 1/99 to 25/75. To test a larger human fraction, we also train on a 50/50 split on SD1.5 and PixArt-$\Sigma$-512 (\cref{tab:5050}). \Approach~still outperforms naive mixing on both backbones, with PickScore gains of $+0.15$ and $+0.07$ and a head-to-head ImageReward@Top-1 win rate of $64.47\%$ and $61.15\%$. This is consistent with the mechanism of the correction, whose benefit scales with the amount of biased synthetic supervision in the pool and vanishes only as the human fraction approaches $100\%$ or as the labeler becomes unbiased (\cref{prop:unbiasedness}).

\begin{table}[h]
  \centering
  \small
  \caption{Results with a 50/50 human/synthetic split. ImageReward@Top-1 is computed head-to-head between the two methods of each backbone.}
  \label{tab:5050}
  \begin{tabular}{llcc}
    \toprule
    \textbf{Model} & \textbf{Method} & \textbf{PickScore} ($\uparrow$) & \textbf{ImageReward@Top-1 (\%)} ($\uparrow$) \\
    \midrule
    \multirow{2}{*}{SD1.5} & DPO+50-50 & 21.70 & 35.53 \\
     & \Approach+50-50 & \textbf{21.85} & \textbf{64.47} \\
    \midrule
    \multirow{2}{*}{PixArt-$\Sigma$-512} & DPO+50-50 & 22.64 & 38.85 \\
     & \Approach+50-50 & \textbf{22.71} & \textbf{61.15} \\
    \bottomrule
  \end{tabular}
\end{table}

\subsection{Behavior with strong synthetic labelers}
\label{app:strong_labeler}
By \cref{prop:unbiasedness}, if the synthetic labels equal the human labels pointwise, $\hat\gG(\vy)=z$, the correction term vanishes and \Approach~reduces to standard training on all $n_l+n_u$ pairs, so the expected benefit shrinks as the labeler approaches the human target. This matches our observations: the gains are smaller on HPD-5K (about $93\%$ Qwen--human agreement) than on PaP-5K (about $83\%$). Because these two agreement levels come from different datasets, \cref{fig:labeler_sensitivity} provides the controlled comparison on a single dataset (PaP-5K), varying labeler quality from roughly $50\%$ to $90\%$ agreement; \Approach~stays above DPO across this range, with the largest gap under the most biased labelers. A strong but imperfect labeler can still make a small set of systematic errors, which is exactly what the correction targets. We do not claim a finite-sample no-harm guarantee: estimating the correction can add variance, and the theoretical statement is unbiasedness in expectation together with exact reduction to standard training when the labels are perfect. Empirically, performance remains stable in the high-agreement regime, e.g., the small but non-negative HPD-5K margins in \cref{tab:main}.

\section{Limitations}
Our work has several limitations. First, although DeDPO consistently improves over matched semi-supervised DPO baselines, its gains still depend on the quality and structure of the synthetic labeler. When synthetic preferences are only weakly aligned with human judgment, or when their errors are highly instance-specific rather than systematic, the correction induced by the small labeled subset may be less effective. Second, our strongest empirical results are obtained in the label-scarce offline setting studied here, with fixed pretrained synthetic annotators and limited backbone diversity. We do not claim that the same gains will transfer unchanged to all diffusion architectures, prompt distributions, or online preference-optimization pipelines. Third, the theoretical guarantees rely on assumptions such as sample splitting, smoothness, and sufficiently well-behaved function classes; these assumptions clarify the mechanism behind debiasing, but they do not perfectly capture every practical training regime, especially self-training settings that violate independence between the synthetic labeler and the optimized model. Finally, our evaluation relies primarily on automatic preference metrics and qualitative inspection. Although we use stronger human-preference-aligned evaluators than older proxy metrics, these remain imperfect surrogates for large-scale human studies. A useful direction for future work is therefore to extend DeDPO to broader architectures and feedback sources, and to validate it with more direct human evaluation under diverse prompting conditions.

\newpage
\section*{NeurIPS Paper Checklist}

\begin{enumerate}

  \item {\bf Claims}
  \item[] Question: Do the main claims made in the abstract and introduction accurately reflect the paper's contributions and scope?
  \item[] Answer: \answerYes{} %
  \item[] Guidelines:
    \begin{itemize}
      \item The answer \answerNA{} means that the abstract and introduction do not include the claims made in the paper.
      \item The abstract and/or introduction should clearly state the claims made, including the contributions made in the paper and important assumptions and limitations. A \answerNo{} or \answerNA{} answer to this question will not be perceived well by the reviewers.
      \item The claims made should match theoretical and experimental results, and reflect how much the results can be expected to generalize to other settings.
      \item It is fine to include aspirational goals as motivation as long as it is clear that these goals are not attained by the paper.
    \end{itemize}

  \item {\bf Limitations}
  \item[] Question: Does the paper discuss the limitations of the work performed by the authors?
  \item[] Answer: \answerYes %
  \item[] Justification: We already discuss the limitaton in the appendix.
  \item[] Guidelines:
    \begin{itemize}
      \item The answer \answerNA{} means that the paper has no limitation while the answer \answerNo{} means that the paper has limitations, but those are not discussed in the paper.
      \item The authors are encouraged to create a separate ``Limitations'' section in their paper.
      \item The paper should point out any strong assumptions and how robust the results are to violations of these assumptions (e.g., independence assumptions, noiseless settings, model well-specification, asymptotic approximations only holding locally). The authors should reflect on how these assumptions might be violated in practice and what the implications would be.
      \item The authors should reflect on the scope of the claims made, e.g., if the approach was only tested on a few datasets or with a few runs. In general, empirical results often depend on implicit assumptions, which should be articulated.
      \item The authors should reflect on the factors that influence the performance of the approach. For example, a facial recognition algorithm may perform poorly when image resolution is low or images are taken in low lighting. Or a speech-to-text system might not be used reliably to provide closed captions for online lectures because it fails to handle technical jargon.
      \item The authors should discuss the computational efficiency of the proposed algorithms and how they scale with dataset size.
      \item If applicable, the authors should discuss possible limitations of their approach to address problems of privacy and fairness.
      \item While the authors might fear that complete honesty about limitations might be used by reviewers as grounds for rejection, a worse outcome might be that reviewers discover limitations that aren't acknowledged in the paper. The authors should use their best judgment and recognize that individual actions in favor of transparency play an important role in developing norms that preserve the integrity of the community. Reviewers will be specifically instructed to not penalize honesty concerning limitations.
    \end{itemize}

  \item {\bf Theory assumptions and proofs}
  \item[] Question: For each theoretical result, does the paper provide the full set of assumptions and a complete (and correct) proof?
  \item[] Answer: \answerYes{} %
  \item[] Justification: We provide the full proof in the appendix
  \item[] Guidelines:
    \begin{itemize}
      \item The answer \answerNA{} means that the paper does not include theoretical results.
      \item All the theorems, formulas, and proofs in the paper should be numbered and cross-referenced.
      \item All assumptions should be clearly stated or referenced in the statement of any theorems.
      \item The proofs can either appear in the main paper or the supplemental material, but if they appear in the supplemental material, the authors are encouraged to provide a short proof sketch to provide intuition.
      \item Inversely, any informal proof provided in the core of the paper should be complemented by formal proofs provided in appendix or supplemental material.
      \item Theorems and Lemmas that the proof relies upon should be properly referenced.
    \end{itemize}

  \item {\bf Experimental result reproducibility}
  \item[] Question: Does the paper fully disclose all the information needed to reproduce the main experimental results of the paper to the extent that it affects the main claims and/or conclusions of the paper (regardless of whether the code and data are provided or not)?
  \item[] Answer: \answerYes %
  \item[] Justification: See the experimental setup.
  \item[] Guidelines:
    \begin{itemize}
      \item The answer \answerNA{} means that the paper does not include experiments.
      \item If the paper includes experiments, a \answerNo{} answer to this question will not be perceived well by the reviewers: Making the paper reproducible is important, regardless of whether the code and data are provided or not.
      \item If the contribution is a dataset and\slash or model, the authors should describe the steps taken to make their results reproducible or verifiable.
      \item Depending on the contribution, reproducibility can be accomplished in various ways. For example, if the contribution is a novel architecture, describing the architecture fully might suffice, or if the contribution is a specific model and empirical evaluation, it may be necessary to either make it possible for others to replicate the model with the same dataset, or provide access to the model. In general. releasing code and data is often one good way to accomplish this, but reproducibility can also be provided via detailed instructions for how to replicate the results, access to a hosted model (e.g., in the case of a large language model), releasing of a model checkpoint, or other means that are appropriate to the research performed.
      \item While NeurIPS does not require releasing code, the conference does require all submissions to provide some reasonable avenue for reproducibility, which may depend on the nature of the contribution. For example
        \begin{enumerate}
          \item If the contribution is primarily a new algorithm, the paper should make it clear how to reproduce that algorithm.
          \item If the contribution is primarily a new model architecture, the paper should describe the architecture clearly and fully.
          \item If the contribution is a new model (e.g., a large language model), then there should either be a way to access this model for reproducing the results or a way to reproduce the model (e.g., with an open-source dataset or instructions for how to construct the dataset).
          \item We recognize that reproducibility may be tricky in some cases, in which case authors are welcome to describe the particular way they provide for reproducibility. In the case of closed-source models, it may be that access to the model is limited in some way (e.g., to registered users), but it should be possible for other researchers to have some path to reproducing or verifying the results.
        \end{enumerate}
    \end{itemize}

  \item {\bf Open access to data and code}
  \item[] Question: Does the paper provide open access to the data and code, with sufficient instructions to faithfully reproduce the main experimental results, as described in supplemental material?
  \item[] Answer: \answerYes %
  \item[] Justification: The full code will be released upon acceptance. We already provide the pseudo code in appendix.
  \item[] Guidelines:
    \begin{itemize}
      \item The answer \answerNA{} means that paper does not include experiments requiring code.
      \item Please see the NeurIPS code and data submission guidelines (\url{https://neurips.cc/public/guides/CodeSubmissionPolicy}) for more details.
      \item While we encourage the release of code and data, we understand that this might not be possible, so \answerNo{} is an acceptable answer. Papers cannot be rejected simply for not including code, unless this is central to the contribution (e.g., for a new open-source benchmark).
      \item The instructions should contain the exact command and environment needed to run to reproduce the results. See the NeurIPS code and data submission guidelines (\url{https://neurips.cc/public/guides/CodeSubmissionPolicy}) for more details.
      \item The authors should provide instructions on data access and preparation, including how to access the raw data, preprocessed data, intermediate data, and generated data, etc.
      \item The authors should provide scripts to reproduce all experimental results for the new proposed method and baselines. If only a subset of experiments are reproducible, they should state which ones are omitted from the script and why.
      \item At submission time, to preserve anonymity, the authors should release anonymized versions (if applicable).
      \item Providing as much information as possible in supplemental material (appended to the paper) is recommended, but including URLs to data and code is permitted.
    \end{itemize}

  \item {\bf Experimental setting/details}
  \item[] Question: Does the paper specify all the training and test details (e.g., data splits, hyperparameters, how they were chosen, type of optimizer) necessary to understand the results?
  \item[] Answer: \answerYes %
  \item[] Justification: See the experimental setup
  \item[] Guidelines:
    \begin{itemize}
      \item The answer \answerNA{} means that the paper does not include experiments.
      \item The experimental setting should be presented in the core of the paper to a level of detail that is necessary to appreciate the results and make sense of them.
      \item The full details can be provided either with the code, in appendix, or as supplemental material.
    \end{itemize}

  \item {\bf Experiment statistical significance}
  \item[] Question: Does the paper report error bars suitably and correctly defined or other appropriate information about the statistical significance of the experiments?
  \item[] Answer: \answerYes %
  \item[] Justification: See the main results
  \item[] Guidelines:
    \begin{itemize}
      \item The answer \answerNA{} means that the paper does not include experiments.
      \item The authors should answer \answerYes{} if the results are accompanied by error bars, confidence intervals, or statistical significance tests, at least for the experiments that support the main claims of the paper.
      \item The factors of variability that the error bars are capturing should be clearly stated (for example, train/test split, initialization, random drawing of some parameter, or overall run with given experimental conditions).
      \item The method for calculating the error bars should be explained (closed form formula, call to a library function, bootstrap, etc.)
      \item The assumptions made should be given (e.g., Normally distributed errors).
      \item It should be clear whether the error bar is the standard deviation or the standard error of the mean.
      \item It is OK to report 1-sigma error bars, but one should state it. The authors should preferably report a 2-sigma error bar than state that they have a 96\% CI, if the hypothesis of Normality of errors is not verified.
      \item For asymmetric distributions, the authors should be careful not to show in tables or figures symmetric error bars that would yield results that are out of range (e.g., negative error rates).
      \item If error bars are reported in tables or plots, the authors should explain in the text how they were calculated and reference the corresponding figures or tables in the text.
    \end{itemize}

  \item {\bf Experiments compute resources}
  \item[] Question: For each experiment, does the paper provide sufficient information on the computer resources (type of compute workers, memory, time of execution) needed to reproduce the experiments?
  \item[] Answer: \answerYes %
  \item[] Justification: See the experimental setup
  \item[] Guidelines:
    \begin{itemize}
      \item The answer \answerNA{} means that the paper does not include experiments.
      \item The paper should indicate the type of compute workers CPU or GPU, internal cluster, or cloud provider, including relevant memory and storage.
      \item The paper should provide the amount of compute required for each of the individual experimental runs as well as estimate the total compute.
      \item The paper should disclose whether the full research project required more compute than the experiments reported in the paper (e.g., preliminary or failed experiments that didn't make it into the paper).
    \end{itemize}

  \item {\bf Code of ethics}
  \item[] Question: Does the research conducted in the paper conform, in every respect, with the NeurIPS Code of Ethics \url{https://neurips.cc/public/EthicsGuidelines}?
  \item[] Answer: \answerYes %
  \item[] Justification: The research conform with NeurIPS CoE.
  \item[] Guidelines:
    \begin{itemize}
      \item The answer \answerNA{} means that the authors have not reviewed the NeurIPS Code of Ethics.
      \item If the authors answer \answerNo, they should explain the special circumstances that require a deviation from the Code of Ethics.
      \item The authors should make sure to preserve anonymity (e.g., if there is a special consideration due to laws or regulations in their jurisdiction).
    \end{itemize}

  \item {\bf Broader impacts}
  \item[] Question: Does the paper discuss both potential positive societal impacts and negative societal impacts of the work performed?
  \item[] Answer: \answerNo{} %
  \item[] Justification: No disinformation or privacy considerations
  \item[] Guidelines:
    \begin{itemize}
      \item The answer \answerNA{} means that there is no societal impact of the work performed.
      \item If the authors answer \answerNA{} or \answerNo, they should explain why their work has no societal impact or why the paper does not address societal impact.
      \item Examples of negative societal impacts include potential malicious or unintended uses (e.g., disinformation, generating fake profiles, surveillance), fairness considerations (e.g., deployment of technologies that could make decisions that unfairly impact specific groups), privacy considerations, and security considerations.
      \item The conference expects that many papers will be foundational research and not tied to particular applications, let alone deployments. However, if there is a direct path to any negative applications, the authors should point it out. For example, it is legitimate to point out that an improvement in the quality of generative models could be used to generate Deepfakes for disinformation. On the other hand, it is not needed to point out that a generic algorithm for optimizing neural networks could enable people to train models that generate Deepfakes faster.
      \item The authors should consider possible harms that could arise when the technology is being used as intended and functioning correctly, harms that could arise when the technology is being used as intended but gives incorrect results, and harms following from (intentional or unintentional) misuse of the technology.
      \item If there are negative societal impacts, the authors could also discuss possible mitigation strategies (e.g., gated release of models, providing defenses in addition to attacks, mechanisms for monitoring misuse, mechanisms to monitor how a system learns from feedback over time, improving the efficiency and accessibility of ML).
    \end{itemize}

  \item {\bf Safeguards}
  \item[] Question: Does the paper describe safeguards that have been put in place for responsible release of data or models that have a high risk for misuse (e.g., pre-trained language models, image generators, or scraped datasets)?
  \item[] Answer: \answerNo{} %
  \item[] Justification: The paper proposed an algorithm which is safe.
  \item[] Guidelines:
    \begin{itemize}
      \item The answer \answerNA{} means that the paper poses no such risks.
      \item Released models that have a high risk for misuse or dual-use should be released with necessary safeguards to allow for controlled use of the model, for example by requiring that users adhere to usage guidelines or restrictions to access the model or implementing safety filters.
      \item Datasets that have been scraped from the Internet could pose safety risks. The authors should describe how they avoided releasing unsafe images.
      \item We recognize that providing effective safeguards is challenging, and many papers do not require this, but we encourage authors to take this into account and make a best faith effort.
    \end{itemize}

  \item {\bf Licenses for existing assets}
  \item[] Question: Are the creators or original owners of assets (e.g., code, data, models), used in the paper, properly credited and are the license and terms of use explicitly mentioned and properly respected?
  \item[] Answer: \answerNo{} %
  \item[] Justification: We use public dataset.
  \item[] Guidelines:
    \begin{itemize}
      \item The answer \answerNA{} means that the paper does not use existing assets.
      \item The authors should cite the original paper that produced the code package or dataset.
      \item The authors should state which version of the asset is used and, if possible, include a URL.
      \item The name of the license (e.g., CC-BY 4.0) should be included for each asset.
      \item For scraped data from a particular source (e.g., website), the copyright and terms of service of that source should be provided.
      \item If assets are released, the license, copyright information, and terms of use in the package should be provided. For popular datasets, \url{paperswithcode.com/datasets} has curated licenses for some datasets. Their licensing guide can help determine the license of a dataset.
      \item For existing datasets that are re-packaged, both the original license and the license of the derived asset (if it has changed) should be provided.
      \item If this information is not available online, the authors are encouraged to reach out to the asset's creators.
    \end{itemize}

  \item {\bf New assets}
  \item[] Question: Are new assets introduced in the paper well documented and is the documentation provided alongside the assets?
  \item[] Answer: \answerNA{} %
  \item[] Justification:
  \item[] Guidelines:
    \begin{itemize}
      \item The answer \answerNA{} means that the paper does not release new assets.
      \item Researchers should communicate the details of the dataset\slash code\slash model as part of their submissions via structured templates. This includes details about training, license, limitations, etc.
      \item The paper should discuss whether and how consent was obtained from people whose asset is used.
      \item At submission time, remember to anonymize your assets (if applicable). You can either create an anonymized URL or include an anonymized zip file.
    \end{itemize}

  \item {\bf Crowdsourcing and research with human subjects}
  \item[] Question: For crowdsourcing experiments and research with human subjects, does the paper include the full text of instructions given to participants and screenshots, if applicable, as well as details about compensation (if any)?
  \item[] Answer: \answerYes{} %
  \item[] Justification: See the discussion in human evaluation.
  \item[] Guidelines:
    \begin{itemize}
      \item The answer \answerNA{} means that the paper does not involve crowdsourcing nor research with human subjects.
      \item Including this information in the supplemental material is fine, but if the main contribution of the paper involves human subjects, then as much detail as possible should be included in the main paper.
      \item According to the NeurIPS Code of Ethics, workers involved in data collection, curation, or other labor should be paid at least the minimum wage in the country of the data collector.
    \end{itemize}

  \item {\bf Institutional review board (IRB) approvals or equivalent for research with human subjects}
  \item[] Question: Does the paper describe potential risks incurred by study participants, whether such risks were disclosed to the subjects, and whether Institutional Review Board (IRB) approvals (or an equivalent approval/review based on the requirements of your country or institution) were obtained?
  \item[] Answer: \answerNA{} %
  \item[] Justification:
  \item[] Guidelines:
    \begin{itemize}
      \item The answer \answerNA{} means that the paper does not involve crowdsourcing nor research with human subjects.
      \item Depending on the country in which research is conducted, IRB approval (or equivalent) may be required for any human subjects research. If you obtained IRB approval, you should clearly state this in the paper.
      \item We recognize that the procedures for this may vary significantly between institutions and locations, and we expect authors to adhere to the NeurIPS Code of Ethics and the guidelines for their institution.
      \item For initial submissions, do not include any information that would break anonymity (if applicable), such as the institution conducting the review.
    \end{itemize}

  \item {\bf Declaration of LLM usage}
  \item[] Question: Does the paper describe the usage of LLMs if it is an important, original, or non-standard component of the core methods in this research? Note that if the LLM is used only for writing, editing, or formatting purposes and does \emph{not} impact the core methodology, scientific rigor, or originality of the research, declaration is not required.
  \item[] Answer: \answerNA{} %
  \item[] Justification:
  \item[] Guidelines:
    \begin{itemize}
      \item The answer \answerNA{} means that the core method development in this research does not involve LLMs as any important, original, or non-standard components.
      \item Please refer to our LLM policy in the NeurIPS handbook for what should or should not be described.
    \end{itemize}

\end{enumerate}

\end{document}